\documentclass[letterpaper, 10 pt, conference]{ieeeconf}  

\IEEEoverridecommandlockouts                              
\makeatletter

\let\proof\@undefined
\let\endproof\@undefined
\let\IEEElabelindent\labelindent
\makeatother
\usepackage{graphicx}
\usepackage{tabularx}
\usepackage{enumerate}
\usepackage{hyperref}
\usepackage{upgreek}
\usepackage{todonotes}
\usepackage{cite}

\makeatletter
\let\MYcaption\@makecaption
\makeatother

\usepackage[font=footnotesize]{subcaption}

\makeatletter
\let\@makecaption\MYcaption
\makeatother

\usepackage{amsmath} 
\usepackage{amssymb}  
\usepackage{amsthm}  
\usepackage{bm}
\usepackage{lipsum}
\usepackage[linesnumbered, ruled]{algorithm2e}
\usepackage{color}
\usepackage{array}
\usepackage{algorithmic}
\usepackage{booktabs}

\newtheorem{proposition}{Proposition}
\newtheorem{definition}{Definition}

\DeclareMathOperator*{\argmin}{arg\,min}

\title{\LARGE \bf
Safe Reinforcement Learning for Probabilistic Reachability and\\ Safety Specifications: A Lyapunov-Based Approach
\thanks{This work was supported in part by  the Creative-Pioneering Researchers Program through SNU, the Basic Research Lab Program through the National Research Foundation of Korea funded by the MSIT(2018R1A4A1059976), and Samsung Electronics.
}
}

\author{
Subin Huh  \and Insoon Yang
\thanks{S. Huh, and I. Yang are with the Department of Electrical and Computer Engineering, Automation and Systems Research Institute,  Seoul National University, Seoul 08826, South Korea,  {\tt\small \{subinh1994, insoonyang\}@snu.ac.kr}}%
}

\begin{document}

\maketitle

\begin{abstract}
	Emerging applications in robotic and autonomous systems, such as autonomous driving and robotic surgery, often involve critical safety constraints that must be satisfied even when information about system models is limited. 
		In this regard, we propose a model-free safety specification method that learns the maximal probability of safe operation by carefully combining probabilistic reachability analysis and safe reinforcement learning (RL). 	
	Our approach constructs a Lyapunov function with respect to a safe policy to restrain each policy improvement stage.
	As a result, it yields a sequence of safe policies that determine the range of safe operation, called the \emph{safe set}, which monotonically expands and gradually converges.
	We also develop an efficient safe exploration scheme that accelerates the process of identifying the safety of unexamined states.
Exploiting the Lyapunov shieding, our method regulates the exploratory policy to avoid dangerous states with high confidence.
To handle high-dimensional systems, we further extend our approach to deep RL by introducing a Lagrangian relaxation technique to establish a tractable actor-critic algorithm.
The empirical performance of our method is demonstrated through continuous control benchmark problems, such as a reaching task on a planar robot arm.
\end{abstract}

\section{Introduction}

Reachability and safety specifications for robotic and autonomous systems
are one of fundamental problems for the verification of such systems. 
It is difficult to imagine deploying robots, without (safety) verification, in practical environments  due to possible critical issues such as collisions and malfunctions.
Several reachability analysis techniques have been developed for the safe operation of various types of systems (e.g.,~\cite{abate2008probabilistic, Majumdar2014, Chen2018}) and applied to quadrotor control~\cite{gillula2011applications}, legged locomotion~\cite{Piovan2015}, obstacle avoidance~\cite{Malone2017}, among others.
However, the practicality of these tools is often limited  because
they require knowledge of system models.
The focus of this work is to develop a model-free reinforcement learning method for specifying reachability and safety  in a probabilistic manner. 

Several learning-based  safety specification methods have recently been proposed for \emph{deterministic} dynamical systems without needing complete information about system models. 
To learn backward reachable sets, Hamilton--Jacobi reachability-based tools were used in conjunction with Gaussian process regression~\cite{Fisac2018} and reinforcement learning~\cite{Fisac2019}. 
As another safety certificate, a region of attraction was estimated using Lyapunov-based reinforcement learning~\cite{Berkenkamp2017} and 
a neural network Lyapunov function~\cite{Richards2018}.
Forward invariance has also been exploited for safety verification by learning control barrier functions~\cite{Wang2018, Taylor2019}.

Departing from these tools for deterministic systems, 
we propose a model-free safety specification method for stochastic systems by carefully combining probabilistic reachability analysis and reinforcement learning.
Specifically, our method aims to learn the maximal probability of avoiding the set of unsafe states.
Several methods have been developed for computing the probability of safety  in various cases
via dynamic programming when the system model is known~\cite{abate2008probabilistic, Summers2010, Lesser2016, Yang2018}.  
To overcome this limitation, our tool uses model-free reinforcement learning for estimating the probability of safety.
We further consider safety guarantees during the learning process so that our scheme runs without frequent intervention of a human supervisor who takes care of safety.
To attain this property, we employ the Lyapunov-based RL framework proposed in \cite{chow2018lyapunov}, where the Lyapunov function takes the form of value functions, and thus safety is preserved in a probabilistic manner through the Bellman recursion.
We revise this safe RL method to enhance its exploration capability.
Note that the purpose of exploration in our method is to enlarge or confirm knowledge about safety, while most safe RL schemes  encourage exploration to find reward-maximizing policies within verified safe regions~\cite{turchetta2016safe,alshiekh2018safe,wachi2018safe}.

The main contributions of this work can be summarized as follows.
First, we propose a safe RL method that specifies the probabilistic safety of a given Markov control system without prior information about the system dynamics.
Our approach yields a sequence of safe and improving policies by imposing the Lyapunov constraint in its policy improvement stage and establishing a Lyapunov function in the policy evaluation stage.
If there is no approximation error, our RL-based safety specification algorithm is guaranteed to run safely throughout the learning process.
In such a case, the safe region determined by our approach also monotonically expands in a stable manner, and eventually converges to the maximal safe set.
Second, we  develop an efficient safe exploration scheme to learn safe or reachable sets in a sample-efficient manner.
Safe policies tend to avoid reaching the borders of safe regions, so the ``learned'' probability of safety at their borders and outside them is likely to be more inaccurate than others.
To mitigate the imbalance of knowledge, we select the least-safe policy to encourage exploration.
This exploratory policy visits less-safe states so that the safe set becomes more accurate or grows faster.
Third, we implement our approach with deep neural networks to alleviate the scalability issue that arises in high-dimensional systems.
Converting the Lyapunov constraints to a regularization term, our approach can be implemented in conventional actor-critic algorithms for deep RL.
We further show that our method outperforms other baseline methods through simulation studies.
%

\section{Background}\label{sec:setup}

We consider an MDP, defined as a tuple $\left( \mathcal{S}, \mathcal{A}, p \right)$, where $\mathcal{S}$ is the set of states, $\mathcal{A}$ is the set of actions, and $p: \mathcal{S} \times \mathcal{A} \times \mathcal{S} \to [0,1]$ is the transition probability function.
We also use the notation $\mathcal{S}_{\mathrm{term}}$ and $\mathcal{S}'$ to represent the set of termination states and non-terminal states, respectively.
Moreover, a (stochastic) Markov policy, $\pi : \mathcal{S} \times \mathcal{A} \to [0,1]$, is a measurable function, and $\pi ( \bm{a} | \bm{s} )$ represents the probability of executing action $\bm{a}$ given state $\bm{s}$.
We also let $\Pi$ denote the set of stochastic Markov policies.

\subsection{Probabilistic Reachability and Safety Specifications}

We consider the problem of specifying the probability that the state of an MDP will not visit a pre-specified \emph{target set} $\mathcal{G} \subseteq S$ before arriving at a terminal state given an initial state $\bm{s}$ and a Markov policy $\pi$.
For our purpose of safety specification, the target set represents the set of \emph{unsafe} states.
The probability of safety given a policy $\pi$ and an initial state $\bm{s}$ is  denoted by $P_{\bm{s}}^{\mathrm{safe}} (\pi)$.
To compute it, we consider the problem of evaluating $1- P_{\bm{s}}^{\mathrm{safe}} (\pi)$, which represents the probability of visiting the target set at least once given an initial state $\bm{s}$ and a Markov policy $\pi$:
\[
	P_{\bm{s}}^{\mathrm{reach}}(\pi) := \mathbb{P}^{\pi}\left( \exists t\in \left\{ 0,\dots,T^{\ast}-1 \right\} \;\mathrm{s.t.}\; s_{t}\in\mathcal{G} | s_{0}=\bm{s} \right),
\]
where $T^{\ast}$ is the first time to arrive at a terminal state.
Note that $P_{\bm{s}}^{\mathrm{reach}}(\pi)$ represents the \emph{probability of unsafety}.
Our goal is to compute the minimal probability of unsafety and specify the following \emph{maximal probabilistic safe set} with tolerance $\alpha \in (0,1)$:
\[
	S^* (\alpha) := \{ \bm{s} \in \mathcal{S} \mid \inf_\pi P_{\bm{s}}^{\mathrm{reach}}(\pi)  \leq \alpha \}.
\]
This set can be used for \emph{safety verification}: If the agent is initialized within $S^* (\alpha)$, we can guarantee safety with probability $1 - \alpha$ by carefully steering the agent; otherwise, it is impossible to do so.

We now express the probability of unsafety as an expected sum of stage-wise costs by using the technique proposed in \cite{Summers2010}.
Let $\mathbf{1}_{C}:\mathcal{S}\mapsto\{0,1\}$ denote the indicator function of set $C \subseteq \mathcal{S}$ so that its value is 1 if $s \in C$; otherwise, 0.
Given a sequence of states $\{s_{0},\dots,s_{t}\}$, we observe that
\begin{align*}
	\prod_{k=0}^{t-1}\mathbf{1}_{\mathcal{G}^{c}}(s_{k}) \mathbf{1}_{\mathcal{G}}(s_{t}) =
	\begin{cases}
		1\quad\mathrm{if} \; s_{0},\dots,s_{t-1}\in\mathcal{G}^{c}, s_{t}\in\mathcal{G}
		\\
		0\quad\mathrm{otherwise}.
	\end{cases}
\end{align*}
It is easily seen that the sum of $\prod_{k=0}^{t-1}\mathbf{1}_{\mathcal{G}^{c}}(s_{k}) \mathbf{1}_{\mathcal{G}}(s_{t})$ along the trajectory is equal to 0 if the trajectory is away from $\mathcal{G}$ and 1 if there exists at least one state $s_{t}$ that is in $\mathcal{G}$.
The probability of unsafety under $\pi$ is then given by
\[
	P_{\bm{s}}^{\mathrm{reach}}(\pi) =
	\mathbb{E}^{\pi} \left[ \sum_{t=0}^{T^{\ast}-1} \prod_{k=0}^{t-1}\mathbf{1}_{\mathcal{G}^{c}}(s_{k}) \mathbf{1}_{\mathcal{G}}(s_{t}) \mid s_{0}=\bm{s} \right].
\]

We introduce an auxiliary state $x_{t}$, which is an indicator of whether a trajectory $\{s_{0},\cdots,s_{t-1}\}$ is fully safe or not.
It is defined as
\[
	\begin{aligned}
		&x_{0} = 1, \quad x_{t} = \prod_{k=0}^{t-1} \mathbf{1}_{\mathcal{G}^{c}}(s_{k}),\quad t \geq 1.
	\end{aligned}
\]
Since
	$x_{t+1}=x_{t}\mathbf{1}_{\mathcal{G}^{c}}(s_{t})$,
$x_{t+1}$ depends solely on $(s_{t}, x_{t})$ and $a_t$, so the Markov property holds with respect to the state pair $(s_{t},x_{t})$.
The problem of computing the minimal probability of unsafety can be formulated as
\begin{equation}\label{opt}
	\inf_{\pi \in \Pi} P_{\bm{s}}^{\mathrm{reach}}(\pi)  = \mathbb{E}^{\pi}  \left[ \sum_{t=0}^{T^{\ast}-1} x_{t}\mathbf{1}_{\mathcal{G}}(s_{t}) \mid (s_{0},x_{0})=(\bm{s},1) \right ],
\end{equation}
which is in the form of the standard optimal control problem.
Let $V^{\ast}: \mathcal{S} \times \{0, 1\} \to \mathbb{R}$ denote the optimal value function of this problem, that is, $V^* (\bm{s}, \bm{x}) := \inf_{\pi \in \Pi} \mathbb{E}^{\pi} [  \sum_{t=0}^{T^{\ast}-1} x_{t}\mathbf{1}_{\mathcal{G}}(s_{t}) \mid (s_{0},x_{0})=(\bm{s},\bm{x})  ]$.
After computing the optimal value function, we can obtain the maximal probabilistic safe set by simple thresholding:
\[
	S^* (\alpha) = \{ \bm{s} \in \mathcal{S} \mid V^* (\bm{s}, 1) \leq \alpha \}.
\]
Note that this set is a superset of 
	$S^\pi (\alpha):= \{ \bm{s} \in \mathcal{S} \mid P_{\bm{s}}^{\mathrm{reach}}(\pi) \leq \alpha \} = \{\bm{s} \in \mathcal{S} \mid V^\pi (\bm{s}, 1)\leq \alpha \}$
for any Markov policy $\pi$, where $V^\pi: \mathcal{S} \times \{0,1\}$ denotes the value function of $\pi$ defined by
$V^\pi (\bm{s}, \bm{x}) := \mathbb{E}^{\pi} [  \sum_{t=0}^{T^{\ast}-1} x_{t}\mathbf{1}_{\mathcal{G}}(s_{t}) \mid (s_{0},x_{0})=(\bm{s},\bm{x})  ]$.
To distinguish $S^\pi (\alpha)$ from $S^* (\alpha)$, we refer to the former as the (probabilistic) safe set under $\pi$.

\subsection{Safe Reinforcement Learning}

Our goal is to compute the minimal probability of unsafety and the maximal probabilistic safe set without the knowledge of state transition probabilities in a \emph{safety-preserving} manner.
We propose an RL algorithm that guarantees the safety of the agent during the learning process for safety specification.
More specifically, the sequence $\{\pi_k\}_{k=0, 1, \ldots}$ generated by the proposed RL algorithm satisfies
\begin{equation}\label{const}
	P_{\bm{s}}^{\mathrm{reach}} (\pi_{k+1}) \leq \alpha \quad \forall \bm{s} \in S^{\pi_k} (\alpha)
\end{equation}
for $k=0, 1, \ldots$.
This constraint ensures that
\[	
	S^{\pi_k} (\alpha) \subseteq S^{\pi_{k+1}} (\alpha),
\]
that is, the probabilistic safe set (given $\alpha$) monotonically expands. We also use the constraint~\eqref{const} to perform \emph{safe exploration} to collect sample data by preserving safety in a probabilistic manner.

\section{Lyapunov-Based Safe Reinforcement Learning for  Safety Specification}\label{sec:method}

To determine the set of safe policies that satisfy \eqref{const}, we adopt the Lyapunov function proposed in \cite{chow2018lyapunov} and enhance the approach to incentivize the agent to explore the state space efficiently. 

Throughout the section, we assume that every terminal state lies in $S^{\ast}(\alpha)$ and that, at all events, an agent arrives at a terminal state in a finite period.
Thus, there exists an integer $m$ such that $\mathbb{P}^\pi (s_{m}\in \mathcal{S}_{\mathrm{term}} ; s_{0}=\bm{s}) > 0$ $\forall \bm{s}\in \mathcal{S}, \forall \pi\in\Pi.$
In Section~\ref{sec:lyapunov_approach} and~\ref{sec:ess}, the state space $\mathcal{S}$ and the action space $\mathcal{A}$ are assumed to be finite. This assumption will be relaxed when discussing the deep RL version in Section~\ref{sec:deeprl}.

Let $\mathcal{T}_{d}^{\pi}$ denote the stationary Bellman operator for the cost function $d( \bm{s}, \bm{x}) := \bm{x}\mathbf{1}_{\mathcal{G}}(\bm{s})$
\begin{equation} \nonumber
	\begin{split}
		&(\mathcal{T}_{d}^{\pi} V)( \bm{s}, \bm{x}) := d(\bm{s}, \bm{x})  \\
		&+ \sum_{\bm{a} \in A}\pi( \bm{a}| \bm{s}, \bm{x})\sum_{\bm{s}' \in S} p(\bm{s}'| \bm{s}, \bm{a}) V(\bm{s}',\bm{x}\mathbf{1}_{\mathcal{G}^{c}}(\bm{s}))
	\end{split}
\end{equation}
for all $(\bm{s}, \bm{x}) \in \mathcal{S}' \times \{0,1\}$, and
\[
	(\mathcal{T}_{d}^{\pi}V)( \bm{s}, \bm{x}) :=  0
\]
for all $(\bm{s}, \bm{x}) \in \mathcal{S}_{\mathrm{term}} \times \{0,1\}$.
Note that $\mathcal{T}_d^\pi$ is an $m$-stage contraction with respect to $\| \cdot \|_\infty$ for all $(\bm{s}, \bm{x}) \in \mathcal{S}' \times \{0,1\}$.

\subsection{Lyapunov Safety Specification}\label{sec:lyapunov_approach}

We adopt the following definition of Lyapunov functions,  proposed in~\cite{chow2018lyapunov}:
\begin{definition}
	A function $L:S\times\{0,1\}\mapsto[0,1]$ is said to be a \emph{Lyapunov function} with respect to a Markov policy $\pi$ if it satisfies the following conditions:
	\begin{subequations}
		\begin{align}
		( \mathcal{T}_{d}^{\pi} L)( \bm{s}, \bm{x}) &\leq L(\bm{s}, \bm{x}) \quad
		\forall (\bm{s}, \bm{x}) \in \mathcal{S}\times\{0,1\} \label{condition:lyapunov}
		\\
		L(\bm{s},1) &\leq \alpha
		\quad  \forall \bm{s} \in S_{0}, \label{condition:safety}
		\end{align}
	\end{subequations}
	where $S_{0}$ is a given subset of $S^{\ast}(\alpha)$ and $d(s,x) := x \mathbf{1}_{\mathcal{G}}(s)$.
\end{definition}
Inequalities (\ref{condition:lyapunov}) and (\ref{condition:safety}) are called the \emph{Lyapunov condition} and the \emph{safety condition}, respectively.
We can show that if an arbitrary policy $\tilde{\pi}$ satisfies the Lyapunov condition, then the probability of unsafety at $S_{0}$ does not exceed the threshold $\alpha$.
To see this, we recursively apply $\mathcal{T}_d^\pi$ on both sides of \eqref{condition:safety} and use \eqref{condition:lyapunov} and the monotonicity of $\mathcal{T}_d^\pi$ to obtain that, for any $\bm{s}\in S_{0}$,
\begin{align}\label{eqn:monotonicity}
	&\alpha \geq L(\bm{s},1) \geq ( \mathcal{T}_{d}^{\tilde{\pi}} L)(\bm{s},1) \geq ( ( \mathcal{T}_{d}^{\tilde{\pi}} )^{2} L)(\bm{s},1) \geq \cdots.
\end{align}
 has a unique fixed point, which corresponds to the probability of unsafety, 

Due to the $m$-stage contraction property, $\left( \mathcal{T}_{d}^{\tilde{\pi}} \right)^{m}$ has a unique fixed point that corresponds to the probability of unsafety, $P_{\bm{s}}^{\mathrm{reach}} (\tilde{\pi}) = V^{\tilde{\pi}} (\bm{s}, 1)$, under $\tilde{\pi}$.
Therefore, by the Banach fixed point theorem, we have
\begin{align}\label{eqn:contractivity}
	&\alpha \geq \lim_{k\rightarrow\infty} \left( ( \mathcal{T}_{d}^{\tilde{\pi}})^{km} L  \right )(\bm{s},1) = V^{\tilde{\pi}}(\bm{s},1)
\quad \forall \bm{s}\in S_{0}.
\end{align}

Given a Lyapunov function $L$, consider the set $\{\tilde{\pi} \mid (\mathcal{T}_{d}^{\tilde{\pi}} L)(\bm{s},1) \leq \alpha \; \forall \bm{s} \in S_0 \}$.
Then, any policy $\tilde{\pi}$ in this set satisfies the probabilistic safety condition $P_{\bm{s}}^{\mathrm{reach}}(\tilde{\pi})\leq\alpha$ for all $\bm{s}\in S_{0}$ by \eqref{eqn:contractivity}.
Thus, when $S_0$ is chosen as ${S}^{\pi_k} (\alpha)$, the safety constraint \eqref{const} is satisfied.
This set of safe policies is called the \emph{L-induced policy set}.

We can now introduce the Lyapunov safety specification method.
For iteration $k$, we construct the Lyapunov function $L_{k}$ by using the current policy $\pi_{k}$ and update the policy to $\pi_{k+1}$ taken from the $L_{k}$-induced policy set.
Specifically, we set
\begin{equation*}
	L_{k}(\bm{s}, \bm{x}) := \mathbb{E}^{\pi_{k}} \left[ \sum_{t=0}^{T^{\ast}-1} (d+\epsilon_{k})(s_{t},x_{t}) \mid (s_{0},x_{0})=(\bm{s}, \bm{x}) \right],
\end{equation*}
where $\epsilon_{k}:\mathcal{S}\times\{0,1\}\mapsto\mathbb{R}_{\geq 0}$ is an auxiliary cost function.
Following the cost-shaping method of \cite{chow2018lyapunov}, we define the auxiliary cost as the function 
\[
	\epsilon_{k}(\bm{x}) := \bm{x} \cdot \min_{\bm{s} \in S_{0}} \; \frac{\alpha - V^{\pi_{k}}(\bm{s},1)}{T^{\pi_{k}}(\bm{s},1)},
\]
where $T^{\pi_{k}}(\bm{s},\bm{x})$ is the expected time for an agent to reach $\mathcal{G}$ or $\mathcal{S}_{\mathrm{term}}$ the first time under policy $\pi_{k}$ and initial state $(\bm{s},\bm{x})$.
We refer to $T^{\pi_{k}}(\bm{s},1)$ as the \emph{first-hitting time} for the rest of this article.
It is straightforward to check that the Lyapunov condition~\eqref{condition:lyapunov} is satisfied with $L_k$. 
Furthermore, the function $L_{k}$ satisfies the safety condition \eqref{condition:safety} because, for all $\bm{s} \in S_{0}$,
\begin{equation*}
	\begin{aligned}
		L_{k}(\bm{s},1) &\leq V^{\pi_{k}}(\bm{s},1) + \epsilon_{k}(1)T^{\pi_k}(\bm{s},1)
		\\
		&\leq V^{\pi_{k}}(\bm{s},1) + T^{\pi_k}(\bm{s},1)\cdot\frac{\alpha - V^{\pi_{k}}(\bm{s},1)}{T^{\pi_k}(\bm{s},1)} \leq \alpha.
	\end{aligned}
\end{equation*}
Therefore, $L_k$ is a Lyapunov function.

In the policy improvement step, we select $\pi_{k+1}$ from the $L_k$-induced policy set so the updated policy is both safe and has an expanded probabilistic safe set.
\begin{proposition}\label{prop1}
	Suppose that $\pi_{k+1}$ is chosen in $\{ {\pi} \mid (\mathcal{T}_{d}^{{\pi}} L)(\bm{s},1) \leq \alpha \; \forall \bm{s} \in S^{\pi_k} (\alpha) \}$.
	Then, we have
	\[
		P_{\bm{s}}^{\mathrm{reach}} (\pi_{k+1}) \leq \alpha \quad \forall \bm{s} \in S^{\pi_k} (\alpha),
	\]
	and
	\[
		S^{\pi_k} (\alpha) \subseteq S^{\pi_{k+1}} (\alpha).
	\]
\end{proposition} 
\begin{proof}
	The probabilistic safety of $\pi_{k+1}$ follows from \eqref{eqn:contractivity}.
	This also implies that for an arbitrary $\bm{s} \in S^{\pi_k} (\alpha)$, we have $\bm{s} \in S^{\pi_k} (\alpha)$. Therefore, the result follows.
\end{proof}

To achieve the minimal probability of unsafety, we choose $\pi_{k+1}$ as the ``safest'' one in the $L_{k}$-induced policy, that is, 
\begin{equation}\label{eqn:lyapunov_policy_improvement}
\begin{split}
	&\pi_{k+1}(\cdot| 
	\bm{s})\\
	& \in \argmin_{\pi(\cdot| \bm{s} )} \{ ( \mathcal{T}_{d}^{\pi} V_{k} )(\bm{s}, 1) \mid  (\mathcal{T}_{d}^{\pi} L_{k} ) (\bm{s}, 1) \leq L_{k}(\bm{s}, 1) \}.
	\end{split}
\end{equation}
Note that the value of Lyapunov function is 0 at $\bm{x}=0$, so we need not compute a policy for $\bm{x} = 0$·

As the MDP model is unknown, we approximate the value function of a policy using sample trajectories.
We also use Q-learning to obtain a reliable estimate of state-action value functions.
Let $Q_{V}$ and $Q_{T}$ denote the Q-functions for the probability of unsafety and a first-hitting time, respectively.
Given $(s_t, a_t, s_{t+1})$ obtained by executing $\pi_k$, the Q-functions are updated as follows:
\begin{equation}\label{eqn:qlearning_value_computation}
	\begin{split}
		&\begin{aligned}
			Q_{V}(s_{t},a_{t}) \leftarrow \mathbf{1}_{\mathcal{G}}(s_{t}) + \mathbf{1}_{\mathcal{G}^{c}}(s_{t}) \bigg [ (1-\tau_{l}) Q_{V}(s_{t},a_{t})& 
			\\
			+ \tau_{l} \sum_{\bm{a} \in A}\pi_{k}( \bm{a}|s_{t+1})Q_{V}(s_{t+1}, \bm{a}) \bigg ]&
		\end{aligned}
		\\
		&\begin{aligned}
			Q_{T}(s_{t},a_{t}) \leftarrow \mathbf{1}_{\mathcal{G}^c}(s_{t}) \bigg[&
			\tau_{l} \bigg( 1 + \sum_{\bm{a} \in A}\pi_{k}(\bm{a}|s_{t+1})Q_{T}(s_{t+1}, \bm{a}) \bigg)
			\\
			&+ (1-\tau_{l}) Q_{T}(s_{t},a_{t}) \bigg ],
		\end{aligned}
	\end{split}
\end{equation}
where $\tau_{l}(\bm{s}, \bm{a})$ is the learning rate satisfying $\sum_{l}\tau_{l}(\bm{s}, \bm{a})=\infty$ and $\sum_{l}\tau_{l}^{2}(\bm{s}, \bm{a})<\infty$.
We can also rewrite \eqref{eqn:lyapunov_policy_improvement} as the following linear program associated with Q-functions:
\begin{equation}\label{eqn:qlearning_policy_improvement}
	\begin{aligned}
		\min_{\pi(\cdot | \bm{s})} \;&  \sum_{\bm{a} \in \mathcal{A}} \pi(\bm{a} | \bm{s}) Q_{V,k}(\bm{s}, \bm{a})
		\\
		\mathrm{s.t.} \; & \sum_{\bm{a} \in \mathcal{A}} Q_{L,k}(\bm{s}, \bm{a}) (\pi(\bm{a} | \bm{s}) - \pi_{k}(\bm{a} | \bm{s})) \leq \epsilon_{k},
	\end{aligned}
\end{equation}
where $Q_{L,k}$ is the Q-value of Lyapunov function given by $Q_{L,k}(\bm{s}, \bm{a})=Q_{V,k}(\bm{s}, \bm{a}) + \epsilon_{k}(1) Q_{L,k}(\bm{s}, \bm{a})$ and $\epsilon_{k}$ is the shortened expression of $\epsilon_{k}(1)$.
The policy $\pi_{k+1}(\cdot|\bm{s})$ is then updated as the optimal solution of the linear program \eqref{eqn:qlearning_policy_improvement}.

Combining the policy evaluation and the policy improvement steps of Q-functions, we construct the \emph{Lyapunov safety specification} (LSS) as described in Algorithm~\ref{alg:tabular_lss}.
The convergence property of Q-learning in finite-state, finite-action space is well studied in \cite{tsitsiklis1994asynchronous}, so we omit the theoretical details here.
Under the standard convergence condition for Q-learning, the algorithm obtains a sequence of policies that satisfy Proposition \ref{prop1}.

\begin{algorithm}[t]
	\caption{LSS Q-Learning}
	\label{alg:tabular_lss}
	\begin{algorithmic}[1]
		\REQUIRE Tolerance for unsafety $\alpha \in (0,1)$,\\ baseline policy $\pi_{\mathrm{base}}$;
				\STATE Set initial policy $\pi_{0}$ as $\pi_{\mathrm{base}}$;
		\FOR{each iteration $k$}
		\FOR{each environment step $l$}
		\STATE $a_{t} \sim \pi_{k}(\cdot|s_{t})$
		\STATE Get $s_{t+1} \sim p(\cdot|s_{t},a_{t})$ and $\mathbf{1}_{\mathcal{G}}(s_{t})$;
		\STATE Update $Q_{V}(s_{t},a_{t})$, $Q_{T}(s_{t},a_{t})$ as (\ref{eqn:qlearning_value_computation});
		\STATE Reset the environment if $\mathbf{1}_{\mathcal{G}}(s_{t})=1$;
		\ENDFOR
		\STATE Update $\pi_{k+1}( \cdot | \bm{s})$ by solving \eqref{eqn:qlearning_policy_improvement} for each $\bm{s}$;
		\ENDFOR
	\end{algorithmic}
\end{algorithm}

\subsection{Efficient Safe Exploration}\label{sec:ess}

In this subsection, we develop a novel for safe exploration to efficiently solve a probabilistic safety specification problem.
We can utilize the Lyapunov constraint to construct a policy that takes potentially dangerous actions with adequate probability and thus assures safe navigation.

We take our motivation from the discovery that if a state is falsely assumed to have a high probability of unsafety, it is unlikely to correct the misconception without taking exploratory actions.
Consider the table of Q-value estimates used in the LSS algorithm.
The Q-learning agent is initiated from the blank slate, so it is a safe choice to assume that all unvisited states evolve into the target set with high probability.
As a result, the safe policy derived from the algorithm tends to confine an agent inside the current safe set.
With enough time, the Q-value table becomes accurate at all states, but this is unattainable in practice.
Therefore, it is crucial to explore the unidentified states, and this process involves visiting the exterior of the safe set.

In this regard, we choose the exploratory policy to be the most aggressive among the set of policies that guarantee safety in the safe set.
Conversely, the probabilistic safety of the exploratory policy in the safe set is marginally greater than the tolerance.
As there is no element $\mathcal{G}$ in $S^{\pi_{s}}(\alpha)$, such a policy is likely to bring an agent outside the safe set.
The exploratory policy is efficient if used with an experience replay, the state distribution of which may diverge from the true distribution due to the scarcity of samples obtained in the exterior of the safe set.
Our exploratory policy can mitigate the approximation error due to the discrepancy.

\begin{figure}[!t]
	\centering
	\begin{subfigure}[b]{.9\columnwidth}
		\centering
		\includegraphics[width=\linewidth]{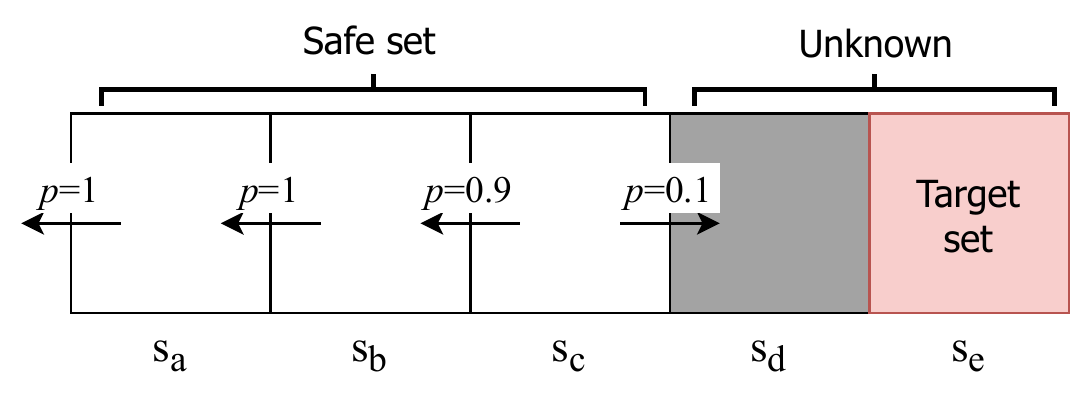}
	\end{subfigure}
	\hfil
	\centering
	\begin{subfigure}[b]{.9\columnwidth}
		\centering
		\includegraphics[width=\linewidth]{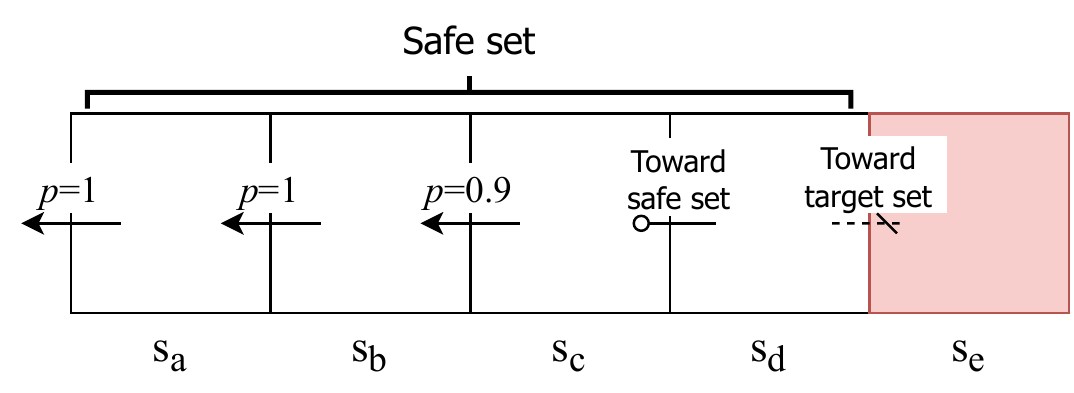}
	\end{subfigure}
	\caption{An example of safe exploration on a one-dimensional grid world. The confidence level is set to $0.9$. Boxes represent states, and arrows toward the left or right symbolize the policies at each state. Unexamined states are shaded; the gray one is not in the target set, but it is considered unsafe. Choosing the policy at $\bm{s}_{c}$ allows an agent to explore toward $\bm{s}_{d}$ (top). As the RL agent successfully returns to the safe set after visiting $\bm{s}_{d}$ with high probability, $\bm{s}_{d}$ is added to the safe set (bottom).}
	\label{fig:ess_example}
\end{figure}

To illustrate our idea, we show a one-dimensional (1D) grid world consisting of five states $\bm{s}_{a}, \dots, \bm{s}_{e}$ and two actions $(\mathrm{left}, \mathrm{right})$ as in Fig. \ref{fig:ess_example}.
We know from experience that moving to the left at $\bm{s}_{a},\dots, \bm{s}_{c}$ guarantees 100\% safety.
The states $\bm{s}_{d}$ and $\bm{s}_{e}$ are not visited yet, so the probabilities of unsafety at those states are 1.
Suppose the agent is in $\bm{s}_{c}$ and chooses to move left or right with probability $(1-\alpha, \alpha)$.
The probability of unsafety of $\pi$ is then no more than $\alpha$ because an agent never reaches $\bm{s}_{d}$ or $\bm{s}_{e}$ with probability $1-\alpha$.
Also, if an agent successfully reaches $\bm{s}_{d}$ or $\bm{s}_{e}$ and returns safely, we obtain an accurate estimate of the probability of unsafety and expand the safe set.

\begin{algorithm}[!t]
	\caption{ESS Q-Learning}
	\label{alg:tabular_ess}
	\begin{algorithmic}
		\REQUIRE Tolerance for unsafety $\alpha \in (0,1)$, \\baseline policy $\pi_{\mathrm{base}}$;
				\STATE Set  $\pi_{s,0} \leftarrow \pi_{\mathrm{base}}$ and $\pi_{e,0}\leftarrow \pi_{\mathrm{base}}$;
		\FOR{each iteration $k$}
		\FOR{each environment step $l$}
		\STATE $a_{t} \sim \pi_{e,k}(\cdot|s_{t})$;
		\STATE Get $s_{t+1} \sim p(\cdot|s_{t},a_{t})$ and $\mathbf{1}_{\mathcal{G}}(s_{t})$;
		\STATE Update $Q_{V}(s_{t},a_{t})$, $Q_{T}(s_{t},a_{t})$ as \eqref{eqn:qlearning_value_computation};
		\STATE Reset environment if $\mathbf{1}_{\mathcal{G}}(s_{t})=1$;
		\ENDFOR
 \STATE Set $\pi_{s,k+1}(\cdot| \bm{s})$  by solving \eqref{eqn:qlearning_policy_improvement} for each $\bm{s}$;
		\STATE Set $\pi_{e,k+1}(\cdot| \bm{s})$ by solving \eqref{eqn:exploratory_policy_improvement} for each $\bm{s}$;
		\ENDFOR
	\end{algorithmic}
\end{algorithm}

A policy suitable for exploration is not usually the safest policy; therefore,  we separate the \emph{exploratory policy} $\pi_{e}$ from the policy that constructs the safe set, which is denoted by the \emph{safety specification-policy} (SS-policy) $\pi_{s}$.
Unlike the SS-policy, the exploratory policy drives an agent around the boundary of the safe set.
To construct $\pi_e$ in a formal way, we exploit a given $\pi_s$ and the Lyapunov function $L$ defined as in Section \ref{sec:lyapunov_approach}.
First, consider the following policy optimization problem:
\begin{equation}\label{def:exploratory_policy_optimization}
	\begin{aligned}
		\max_{\pi\in\Pi}~ & V^{\pi}(s_0,1)\\
		\mathrm{s.t.}~ & (\mathcal{T}_{d}^{\pi} L)(\bm{s}, \bm{x}) \leq L(\bm{s}, \bm{x}) \quad \forall (\bm{s}, \bm{x}) \in \mathcal{S}\times\{0,1\},
	\end{aligned}
\end{equation}
where $s_0$ is an initial state.
Note that this is the auxiliary problem merely to construct the exploratory policy with no connection to the original problem \eqref{opt}.
As stated above, the exploratory policy should preserve safety confidence in the safe set under the SS-policy, that is, $V^{\pi_{e}}(\bm{s},1) \leq \alpha,~\forall \bm{s} \in S^{\pi_{s}}(\alpha)$.
The solution of \eqref{def:exploratory_policy_optimization} satisfies this condition because of the Lyapunov constraint, but it can be suboptimal because the constraint in \eqref{def:exploratory_policy_optimization} is stronger than the original.
However, by using the Lyapunov constraints, we can enjoy the benefit of using dynamic programming to solve \eqref{def:exploratory_policy_optimization}. 

\begin{proposition}\label{prop2}
	Let $L$ be the Lyapunov function stated in \eqref{def:exploratory_policy_optimization}.
An optimal solution of \eqref{def:exploratory_policy_optimization} can be obtained by the value iteration using the Bellman operator
	\begin{equation}\nonumber
	\begin{split}
		&(\mathcal{T}_{\mathrm{exp}} V)(\bm{s}, \bm{x})\\
		& := \max_{\pi(\cdot|\bm{s})} \{ (\mathcal{T}_{d}^{\pi} V)(\bm{s}, \bm{x}) \mid (\mathcal{T}_{d}^{\pi} L)(\bm{s}, \bm{x}) \leq L(\bm{s}, \bm{x}) \}.
	\end{split}
	\end{equation}
	Specifically, the value function that satisfies $\mathcal{T}_{\mathrm{exp}} V = V$ is the probability of unsafety under such a policy.
\end{proposition} 
\begin{proof}
	The operator $\mathcal{T}_{\mathrm{exp}}$ is a special form of the safe Bellman operator defined in \cite{chow2018lyapunov}, which is a monotone contraction mapping by Proposition 3 in \cite{chow2018lyapunov}.
	Thus, there exists a unique fixed point of $\mathcal{T}_{\mathrm{exp}}$.
By the definition of the operator, the fixed point  corresponds to the policy and solves problem \ref{def:exploratory_policy_optimization}.
\end{proof}

As Proposition \ref{prop2} certifies, we can perform the Bellman operation on $V^{\pi_{s}}$ iteratively to obtain $\pi_{e}$, which is the solution of \eqref{def:exploratory_policy_optimization}.
However, in the RL domain, it is difficult to reproduce the whole dynamic programming procedure, since each Bellman operation corresponds to a time-consuming Q-value computation.
We thus apply the Bellman operation once to obtain $\pi_{e}(\cdot|\bm{s})$ at iteration number $k$ as
\begin{equation}\label{eqn:exploratory_policy_improvement}
	\mathrm{arg}\max_{\pi(\cdot| \bm{s})} \{ (\mathcal{T}_{d}^{\pi} V_{k})(\bm{s},1) \mid (\mathcal{T}_{d}^{\pi} L_{k})(\bm{s}, \bm{x}) \leq L_{k}(\bm{s}, \bm{x}) \}.
\end{equation}
To sum up, we add an exploratory policy to LSS to obtain the \emph{exploratory LSS} (ESS), as Algorithm~\ref{alg:tabular_ess}.

\subsection{Deep RL Implementation}\label{sec:deeprl}

Each policy improvement stages in Algorithm \ref{alg:tabular_lss} or \ref{alg:tabular_ess} solves a linear program.
This operation is not straightforward for nontabular implementations. Thus, we provide adaptations of the LSS and ESS for parametrized policies, such as neural networks.
To apply our approach to high-dimensional environments in this section, we assume that the state and action spaces are continuous, which is the general setup in policy gradient (PG) algorithms.
Suppose a generic policy is parameterized with $\theta$, and we rewrite the policy improvement step of the LSS as
\begin{equation}\label{eqn:parameterized_policy_improvement}
	\begin{aligned}
		&\max_{\theta} \int_{\mathcal{A}} - Q_{V}(\bm{s}, \bm{a}) \pi_{\theta}(\bm{a} |\bm{s}) \: \mathrm{d}\bm{a} \quad \mathrm{subject~to}\\
		&\int_{\mathcal{A}} Q_{L}(\bm{s}, \bm{a}) \left( \pi_{\theta}(\bm{a}| \bm{s}) - \pi_{s}(\bm{a}| \bm{s})\right)  \mathrm{d}\bm{a} \leq \epsilon \quad \forall \bm{s}\in \mathcal{S},
	\end{aligned}
\end{equation}
where $\pi_{s}$ is the current SS-policy and $Q_{V}$, $Q_{L}$, and $\epsilon$ are the values defined as the previous section with respect to $\pi_{s}$.

We use Lagrangian relaxation \cite{Bertsekas1999} to form an unconstrained problem.
Ignoring the constraints, the PG minimizes a single objective $\mathbb{E}_{s,a\sim\pi}[Q_{V}(s,a)]$.
The Lyapunov condition is state-wise, so the number of constraints is the same as $|\mathcal{S}|$.
We can replace the constraints with a single one $\max_{\bm{s} \in \mathcal{S}} \int_{\mathcal{A}} Q_{L}(\bm{s},\bm{a}) \left( \pi_{\theta}(\bm{a}|\bm{s}) - \pi_{s}(\bm{a}|\bm{s})\right) \mathrm{d}\bm{a} - \epsilon \leq 0$.
However, one drawback of this formulation is that the Lagrangian multiplier of the max-constraint places excessive weight on the constraint. 
In practice, the LHS of this max-constraint is likely greater than 0 due to the parameterization errors, resulting in the monotonic increase of the Lagrangian multiplier throughout learning.
Therefore, we adopt state-dependent Lagrangian multipliers to have
\begin{equation}\label{eqn:unconstrained_problem}
	\begin{aligned}
		&\min_{\lambda \geq 0} \max_{\theta} \mathbb{E}_{s\sim\rho_{\theta}} \big[ \mathbb{E}_{a\sim\pi_{\theta}}[-Q_{V}(s,a)]\\
		& - \lambda(s) \left(\mathbb{E}_{a\sim\pi_{\theta}}[Q_{L}(s,a)] - \mathbb{E}_{a\sim\pi_{s}}[Q_{L}(s,a)] - \epsilon \right) \big],
	\end{aligned}
\end{equation}
where $\lambda(s)$ is the Lagrangian multiplier at state $s$, and $\rho_{\theta}$ is the discounted state-visiting probability of $\pi_{\theta}$.
We can assume that nearby states have similar $\lambda(s)$. Thus, we can parameterize $\lambda(s)$ as a critic model, as in \cite{bohez2019value}.
Throughout this section, we represent $\omega$ as the parameter of $\lambda$.

Our goal is to find the saddle point of \eqref{eqn:unconstrained_problem}, which is a feasible solution of the original problem \eqref{eqn:parameterized_policy_improvement}.
We apply the gradient descent (ascent) to optimize $\theta$ and $\omega$.
The Q-values that comprise the Lagrangian are, by definition, the functions of the policy parameter $\theta$, but since we incorporate the actor-critic framework, the Q-functions are approximated with critics independent of $\theta$.
In this regard, we obtain the update rules for the safety specification-actor (SS-actor) and the Lagrangian multiplier associated with it as follows:
\begin{subequations}
	\begin{align}
	&\theta_{s} \leftarrow \theta_{s} - \eta_{\theta}\nabla_{\theta} \left( Q_{V}(s_{t},  {a}_{t}) + \lambda_{\omega_{s}}(s_{t}) Q_{L}(s_{t},  {a}_{t}) \right),
	\label{eqn:actor_update}
	\\
	&\begin{aligned}
	\omega_{s} \leftarrow \omega_{s} + \eta_{\omega} \nabla_{\omega}\lambda_{\omega_{s}}(s_{t}) \big( & Q_{L}(s_{t}, {a}_{t}) - \epsilon \\ &- Q_{L}(s_{t}, {a}_{\mathrm{old},t}) \big),
	\end{aligned}
	\label{eqn:lagrangian_update}
	\end{align}
\end{subequations}
where ${a}_{t}\sim\pi_{\theta_{s}}(s_{t})$ and ${a}_{\mathrm{old},t}$ denotes the sampled action from the policy parametrized with the old $\theta_{s}$.

We apply the same approach to improve the exploratory actor.
The unconstrained problem is similar to \eqref{eqn:unconstrained_problem} except for the opposite sign of the primal objective, so we have
\begin{subequations}
	\begin{align}
	&\theta_{e} \leftarrow \theta_{e} + \eta_{\theta}\nabla_{\theta} \left( Q_{V}(s_{t}, {a}_{t}) - \lambda_{\omega_{e}}(s_{t}) Q_{L}(s_{t}, {a}_{t}) \right)
	\label{eqn:exploratory_actor_update}
	\\
	&\begin{aligned}
	\omega_{e} \leftarrow \omega_{e} + \eta_{\omega} \nabla_{\omega}\lambda_{\omega_{e}}(s_{t}) \big( & Q_{L}(s_{t}, {a}_{\mathrm{exp},t}) - \epsilon \\ &- Q_{L}(s_{t},  {a}_{t}) \big),
	\end{aligned}
	\label{eqn:exploratory_lagrangian_update}
	\end{align}
\end{subequations}
where ${a}_{\mathrm{exp},t}\sim\pi_{\theta_{e}}(s_{t})$, ${a}_{t}\sim\pi_{\theta_{s}}(s_{t})$.

Besides, critic parameters are optimized to minimize the Bellman residual.
The scheme is analogous to the Q-learning version, as in \eqref{eqn:qlearning_value_computation}, but in this case, we express the discount factor $\gamma$.
Recall that the Lyapunov Q-function is a weighted sum of the two Q-functions $Q_{V}$ and $Q_{T}$, one for a probability of unsafety and the other for a first-hitting time, respectively.
Letting $\phi$ and $\psi$ represent the parameters of $Q_{V}$ and $Q_{T}$, the targets for the critics $Q_{\phi}$ and $Q_{\psi}$ are defined as
\begin{align*}
&y_{V} := \mathbf{1}_{\mathcal{G}}(s_{t}) + \mathbf{1}_{\mathcal{G}^{c}}(s_{t})\gamma Q_{\phi'}(s_{t+1}, {a}_{t+1})
\\
&y_{T} := \mathbf{1}_{\mathcal{G}^{c}}(s_{t})(1 + \gamma Q_{\psi'}(s_{t+1}, {a}_{t+1})),
\end{align*}
where ${a}_{t+1}$ is the action sampled from $\pi_{\theta_{s}'}(s_{t+1})$.
The proposed actor-critic algorithm is summarized in Algorithm~\ref{alg:actor_critic_lyapunov}.

In our experiments, we use the double Q-learning technique in \cite{hasselt2010double} to prevent the target $y_{V}$ from being overly greater than the true probability of unsafety.
In this case, two critics have independent weights $\phi_{1}$, $\phi_{2}$, and two target critics pertained to the respective critics.
That is, $Q_{\phi'}(s_{t+1}, {a}_{t+1})$ in $y_{V}$ is replaced with $\min_{j=1,2}  Q_{\phi'_{j}}(s_{t+1},{a}_{t+1})$,
where ${a}_{t+1} \sim \pi_{\theta_{s}'}(s_{t+1})$.
Moreover, we adjust the experience replay to alleviate errors in $Q_{V}$.
Catastrophic forgetting is the primary concern, since the target set should be precisely specified to obtain safe policies.
We fix the ratio of safe samples (\textit{i.e.}, $s_{t} \notin \mathcal{G}$) and unsafe samples (\textit{i.e.}, $s_{t} \in \mathcal{G}$) in a minibatch so that the value of $Q_{V}$ is 1 in the identified states of the target set.
We explain the ancilliary techniques in Section \ref{sec:deeprl_result}.

\begin{algorithm}[!t]
	\caption{Actor-critic LSS (ESS)}
	\label{alg:actor_critic_lyapunov}
	\begin{algorithmic}
		\REQUIRE Tolerance for unsafety $\alpha \in (0,1)$; 
		\STATE Initialize SS-actor/critics $\pi_{\theta_{s}}, Q_{\phi}, Q_{\psi}$ and Lagrangian multiplier $\lambda_{\omega_{s}}$;
		\IF {ESS}
		\STATE Initialize exploratory actor $\pi_{\theta_{e}}$ and Lagrangian multiplier $\lambda_{\omega_{e}}$;
		\ENDIF
		\STATE Initialize target networks:
		$\theta_{s}'\leftarrow\theta_{s}$, $\psi'\leftarrow\psi$, $\phi'\leftarrow\phi$;
		\FOR{each iteration $t$}
		\FOR{each environment step}
		\STATE $a_{t} \sim \pi_{\theta_{s}}(\cdot|s_{t})$ (Use $\pi_{\theta_{e}}$ if ESS);
		\STATE $s_{t+1} \sim p(\cdot|s_{t},a_{t})$;
		\STATE $\mathcal{D} \leftarrow \mathcal{D} \cup \{s_{t},a_{t},\mathbf{1}_{\mathcal{G}}(s_{t}),s_{t+1}\}$;
		\STATE Reset environment if $\mathbf{1}_{\mathcal{G}}(s_{t})=1$;
		\ENDFOR
		\FOR{each gradient step}
		\STATE Update $\phi$ by minimizing $\left(y_{V} - Q_{\phi}(s_{t}, a_{t}) \right)^{2}$;
		\STATE Update $\psi$ by minimizing $\left(y_{T} - Q_{\psi}(s_{t}, a_{t}) \right)^{2}$; 
		\STATE Update $\theta_{s}$ as the solution of \eqref{eqn:actor_update};
		\STATE Update $\omega_{s}$ as the solution of \eqref{eqn:lagrangian_update};
		\IF {ESS}
		\STATE Update $\theta_{e}$ as the solution of \eqref{eqn:exploratory_actor_update};
		\STATE Update $\omega_{e}$ as the solution of \eqref{eqn:exploratory_lagrangian_update};
		\ENDIF
		\ENDFOR
		\STATE Soft target update for SS-actor/critic: $\theta_{s}' \leftarrow \tau\theta_{s} + (1-\tau)\theta_{s}'$, $\phi' \leftarrow \tau \phi + (1-\tau)\phi'$, $\psi' \leftarrow \tau\psi+(1-\tau)\psi'$;
		\ENDFOR
	\end{algorithmic}
\end{algorithm}

\section{Simulation Studies}

In this section, we demonstrate our safe learning and safety specification methods using simulated control tasks.
We test the validity of our approach in a simple double integrator and further verify our deep RL algorithms with the high-dimensional dynamic system introduced in \cite{duan2016benchmarking}, both of which have a tuple of positions and velocities as a state.
To make the environments congruous with our problem setting, a target set is defined as the exterior of a certain bounded region of the state space, a setup that enables the implementation of tabular Q-learning.
The description of environments, including the definition of the target sets, can be found in Appendix B.

In Section \ref{sec:method}, we stated the theoretical guarantees as follows.
First, Lyapunov-based methods obtain a subset of $S^{\ast}(\alpha)$.
Second, the improved safe set includes the current safe set.
Third, the agent ensures safety while running in the environment if the initial state is safe.
However, in practice, these guarantees cannot be strictly satisfied, since we determine a safe set with the approximated probability of unsafety.
To distinguish the obtainable safe set from the ideal one derived from the true MDP, we represent the estimate of the safe set under $\pi$ as
\[
	\hat{S}^{\pi}(\alpha) := \{ \bm{s} \in \mathcal{S} : \pi_{s}(\cdot | \bm{s})\hat{Q}_{V}(\bm{s},\cdot) \leq \alpha \}.
\]
We introduce two metrics to quantify how close well-trained RL agents are to such guarantees.
Regarding the accuracy of safety specification, we inspect if a safe set contains the elements of $S^{\ast}(\alpha)$ and if it does not include the unreliable states in $S^{\ast}(\alpha)^{c}$.
We thus consider the \emph{ratio of correct specification}
\[
	r_{\mathrm{c}} = |\hat{S}^{\pi}(\alpha) \cap S^{\ast}(\alpha)| / |S^{\ast}(\alpha)|,
\]
and the \emph{ratio of false-positive specification}
\[
	r_{\mathrm{fp}} = |\hat{S}^{\pi}(\alpha) \cap S^{\ast}(\alpha)^{c}| / |\mathcal{S}|.
\]
We also verify safe exploration throughout learning by tracking the proportion of safely terminated episodes among the 100 latest episodes, which is denoted by the \emph{average episode safety} (AES).
A trajectory is considered safe if an agent reaches terminal states without visiting $\mathcal{G}$ or stays in $\mathcal{G}^{c}$ until the time limit.

Throughout our experiments, we set $\alpha = 0.2$, so AES should be no less than $0.8$ to guarantee safe navigation. We also improve  learning speed by introducing
a discount factor $\gamma < 1$, which is equivalent to $p(s_{t+1}\in \mathcal{S}_{\mathrm{term}}|s_{t},a_{t})$.   As the key idea of our approach is the separation of the exploratory policy from the SS-policy, we set an unmodified RL method as baseline;
that is, the baseline agents are trained to minimize the expected sum of $x_{t}\mathbf{1}_{\mathcal{G}}(s_{t})$.

\subsection{Tabular Q-Learning}

First, we evaluate our Lyapunov-based safety specification methods with tabular implementations.
For tabular Q-learning, we discretize a continuous action $a=(a_{1},\cdots,a_{\dim{\mathcal{A}}})$ into partitions of $A_{1},\cdots,A_{\dim{\mathcal{A}}}$ equal intervals for each element.
In other words, applying the $n$th action for $a_{m}$ is interpreted as $a=(a_{m,\mathrm{max}} - a_{m,\mathrm{min}})\frac{n}{A_{m}-1} + a_{m,\mathrm{min}}$.
Likewise, state space is represented as a finite-dimensional grid.
Based on the MDP quantized as above, the true probability of safety is computed via dynamic programming.

We use a double integrator environment to test the tabular cases.
To reduce the size of $S^{\ast}(\alpha)$, we modify the integrator to perturb the input acceleration with a certain probability.
We compare LSS, ESS, and a baseline Q-learning with no extra techniques to shield unsafe actions.
We initialize the Q-function tables with random values uniformly sampled from the interval $[0.99,1]$; that is, the probability of unsafety is estimated as nearly 1 in all states.
Therefore, in the tabular setting we impose the assumption that all unvisited states have the probabilistic safeties lower than the threshold.
We then perform the policy improvement $10^{2}$ times, each of which comes after $10^{6}$ environment steps.

\begin{figure}[!t]
	\centering
	\begin{subfigure}[b]{.48\columnwidth}
		\centering
		\includegraphics[width=\linewidth]{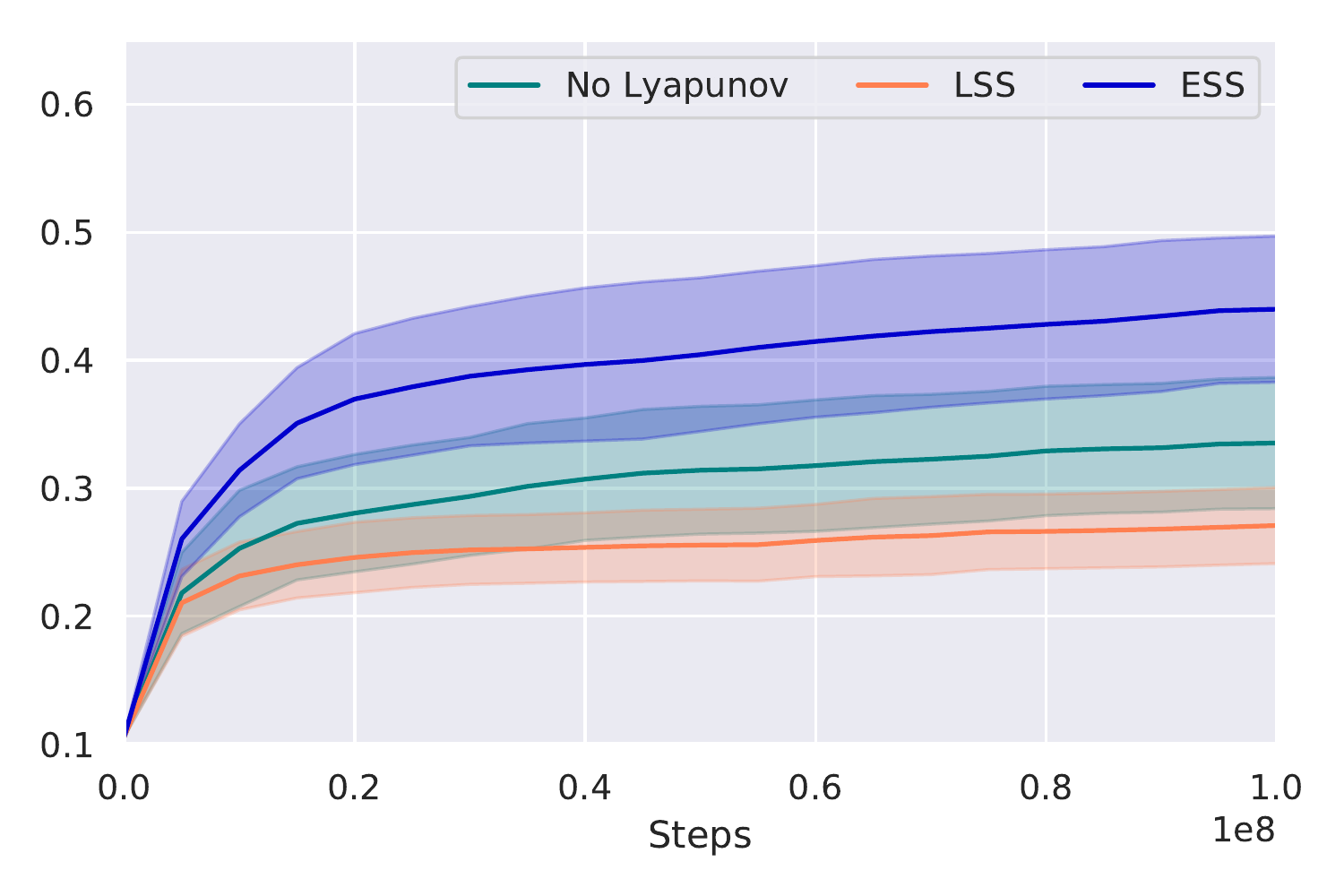}
		\caption{Integrator, $r_{\mathrm{c}}$}
	\end{subfigure}%
	\hfil
	\begin{subfigure}[b]{.48\columnwidth}
		\centering
		\includegraphics[width=\linewidth]{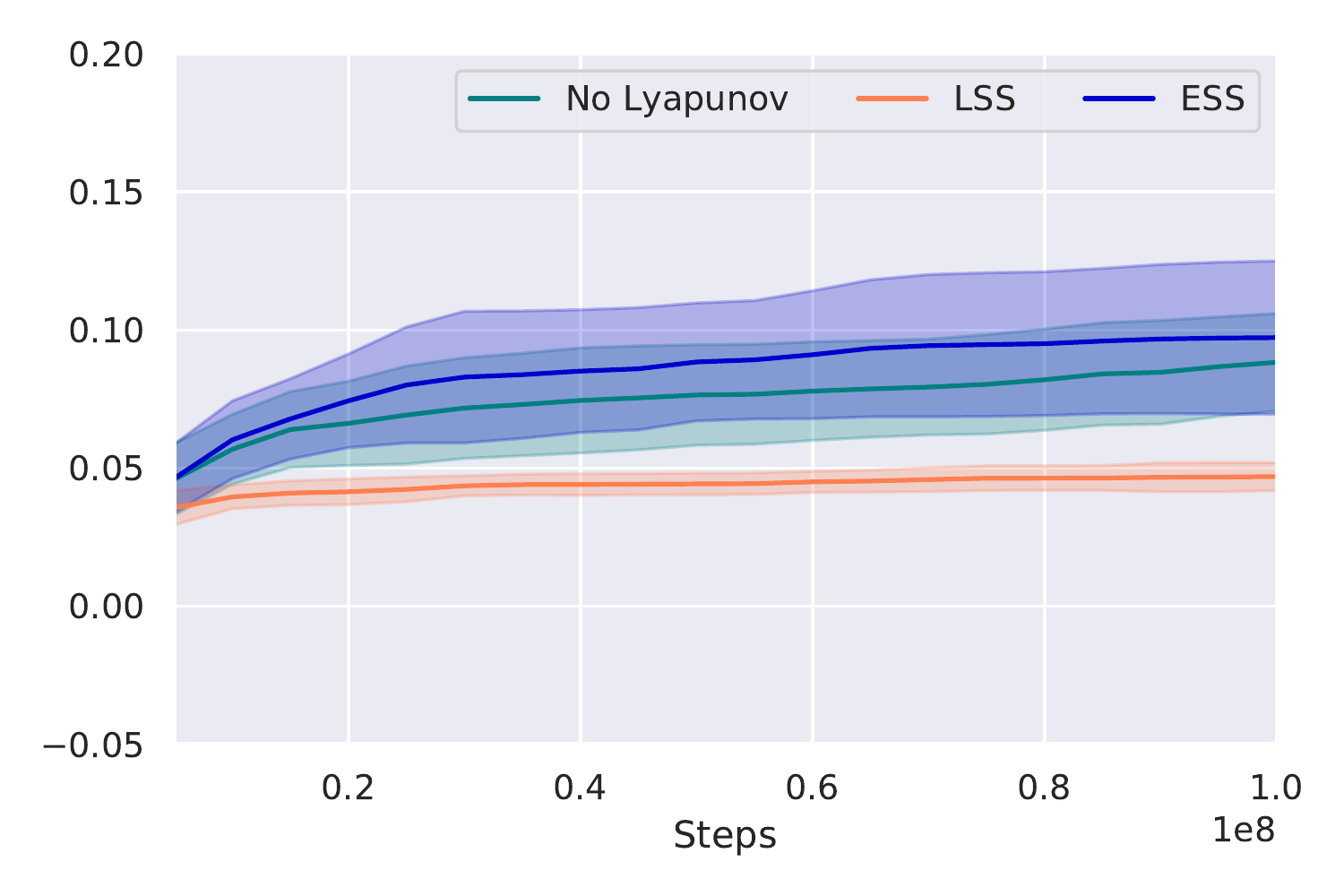}
		\caption{Integrator, $r_{\mathrm{fp}}$}
	\end{subfigure}
	\caption{Safety specification via tabular Q-learning tested on the double integrator. The solid line denotes the average, and the shaded area indicates the confidence interval of 20 random seeds. The baseline, LSS, and ESS are denoted by teal, orange, and blue, respectively.}
	\label{fig:speculation_tab}
\end{figure}

Fig. \ref{fig:speculation_tab} summarizes the specification result averaged across 20 random seeds.
Both LSS and ESS  show monotonic improvement of the approximated safe set $\hat{S}^{\pi}(\alpha)$.
Notably, we find evidence of ESS taking advantage of exploration.
The $r_{\mathrm{c}}$ of ESS increase faster than those of LSS or the baseline, while the excess of $r_{\mathrm{fp}}$ of ESS is negligible.
The average value of $r_{\mathrm{c}}$ is $44\%$ for ESS, surpassing the baseline of $34\%$.
The effect of ESS culminates at the beginning of the learning process then dwindles because the boundary of $\hat{S}^{\pi}(\alpha)$ becomes unlikely to reach as the set inflates, so the chance of exploration decreases.
Ideally, with the appropriate choice of $\gamma \approx 1$ and the learning rate, $r_{\mathrm{fp}}$ is nearly 0.
We skip the AES in Fig. \ref{fig:speculation_tab}, since no agent lacks safety confidence.
However, the AES might decline without the limit, since an episode is configured to terminate after 200 steps, which restricts the chance of reaching the target set.

\begin{figure}[!t]	
\centering
\begin{subfigure}[b]{.27\columnwidth}
		\centering
		\includegraphics[width=\linewidth]{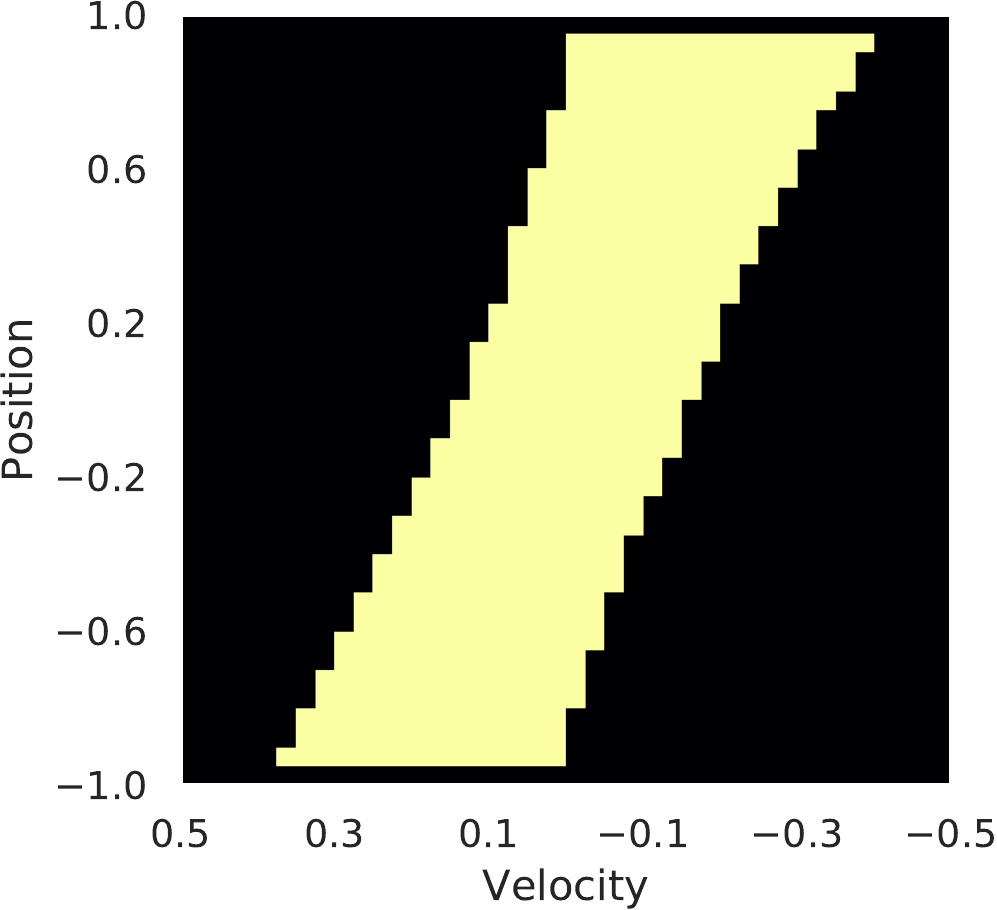}
		\caption{Ground truth}
	\end{subfigure}%
		\hspace*{\fill}
\begin{subfigure}[b]{.23\columnwidth}
		\centering
		\includegraphics[width=\linewidth]{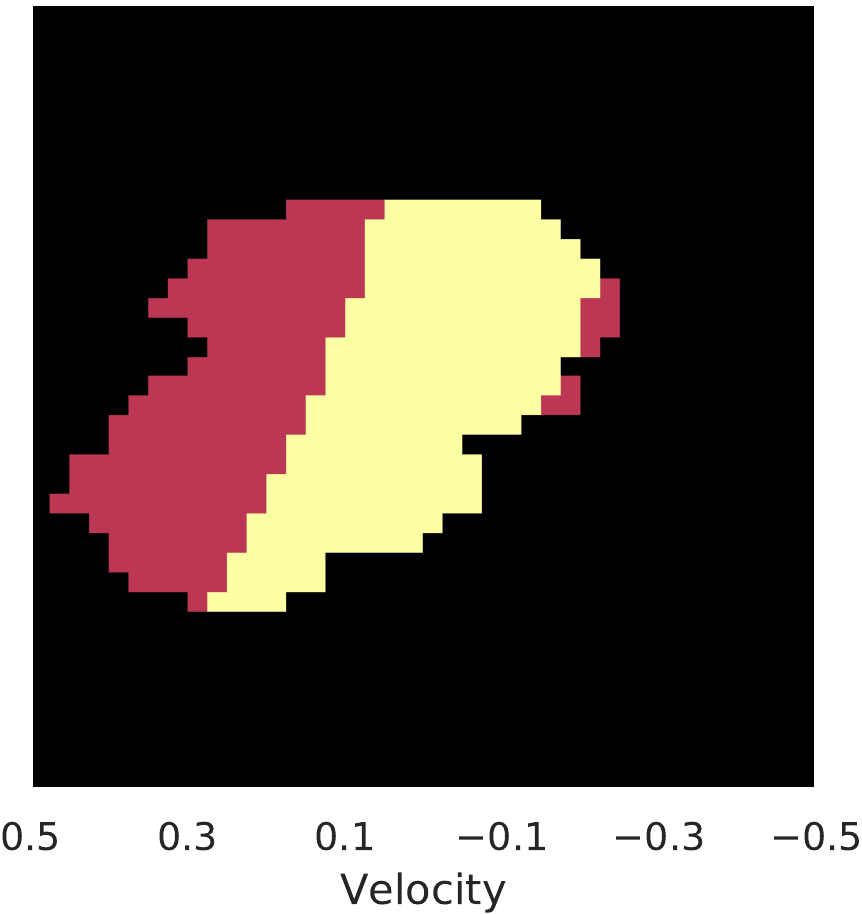}
		\caption{Baseline}
\end{subfigure}	
\begin{subfigure}[b]{.23\columnwidth}
		\centering
		\includegraphics[width=\linewidth]{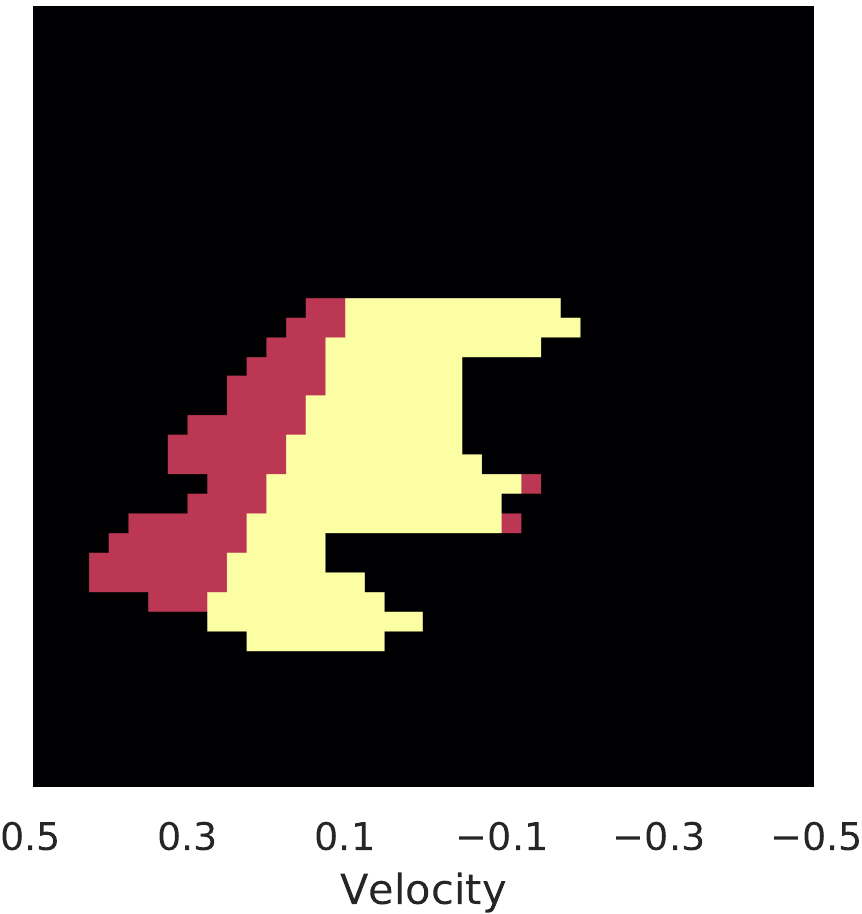}
		\caption{LSS}
\end{subfigure}	
\begin{subfigure}[b]{.23\columnwidth}
		\centering
		\includegraphics[width=\linewidth]{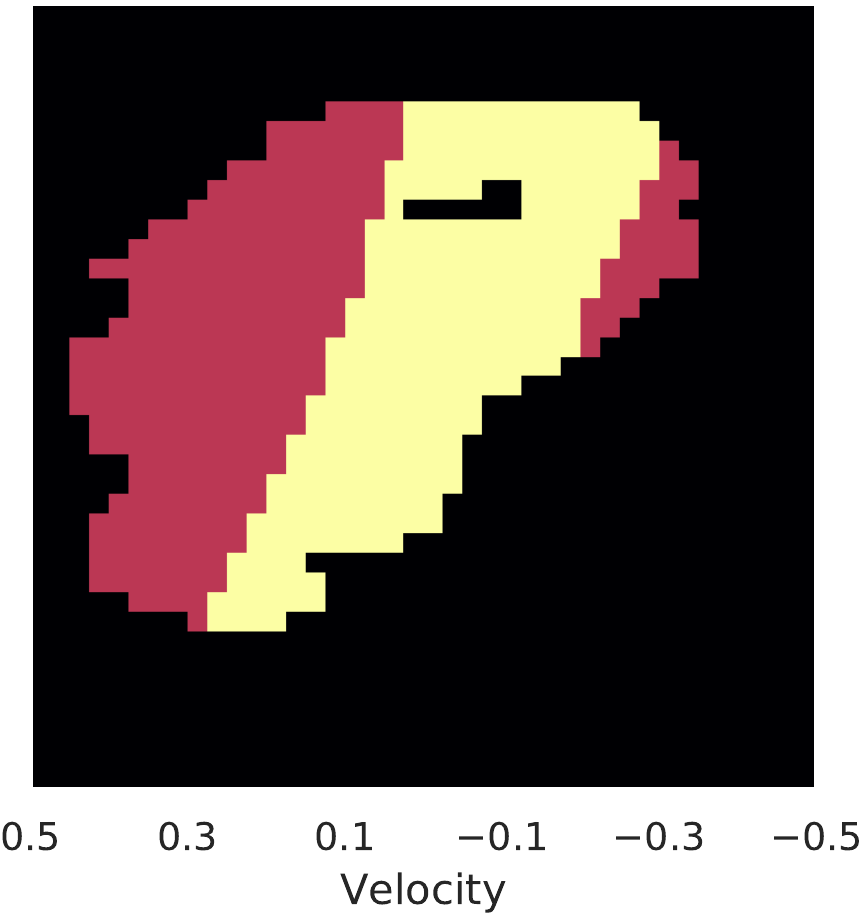}
		\caption{ESS}
\end{subfigure}	
	\caption{Safe sets for the integrator problem with $\alpha = 0.2$. Each grid point denotes a state $(\mathrm{position}, \mathrm{velocity})$. The ground truth $S^{\ast}(\alpha)$ is denoted by yellow in (a). The other figures show the safe set estimated by (b) the baseline, (c) LSS, and (d) ESS. The shaded region represents $\hat{S}^{\pi}(\alpha)$: correctly specified states are marked yellow, and unsafe states misclassified as safe are marked red. }
	\label{fig:visualize_tab}
\end{figure}

We illustrate the safety analysis results of respective methods and the ground-truth probabilistic safe set in Fig.~\ref{fig:visualize_tab}.
Each approximated safe set is established from the Q-learning table of an agent with the highest rate of correct specification among the 20 random seeds analyzed in Fig.~\ref{fig:speculation_tab}.
A grid map represents the whole non-target set except for the grid points on the sides, and the approximated safe set is the set of red and yellow points.
The size of $\hat{S}^{\pi}(\alpha)$ for ESS is notably larger than that of the baseline or LSS in the cases of both correctly specified states (yellow) and misclassified states (red).
However, the false-positive in the  safe set estimated by ESS is hardly due to the ESS method but comes from a universal limitation of tabular Q-learning.
This can be explained from the observation that the ratio of misclassified states over the whole $\hat{S}^{\pi}(\alpha)$ of ESS is greater than that of the baseline only by $5\%$;
that is,  ESS  does not particularly overestimate the probability of safety in unsafe states.
The ESS Q-learning is expected to obtain an accurate estimate of $S^{\ast}$ if the implementation of Q-learning is improved.

\subsection{Deep RL}\label{sec:deeprl_result}

We present the experimental results in Algorithm \ref{alg:actor_critic_lyapunov} using a realistic robotic simulation.
We demonstrate that our approach can be coupled with well-established deep RL methods to perform safety specifications efficiently in the continuous state and action space.
Details about our deep RL implementation can be found in Appendix~A.
We consider a \emph{Reacher} system for safety analysis.
In the Reacher, safety constraints are set on the position of the end effector (See Appendix B for details).

We implement the LSS and ESS actor-critic in DDPG \cite{lillicrap2015continuous}, and the baseline.
For the sake of fairness, all the algorithms use the same actor network weight and the same replay memory at the start of learning.
The critics are initialized randomly, but the bias value for each layer of $Q_{V}$ is set to 1 so that $Q_{V}(\bm{s}, \bm{a}) = 1$ for almost all $(\bm{s}, \bm{a}) \in \mathcal{S} \times \mathcal{A}$.
This ensures that the ratio of correct specification is 0 at the very beginning.
We also optimize only the critics for the first $10^{5}$ steps to reduce the discrepancies between critics and actors.
The techniques mentioned in Section \ref{sec:deeprl} are also applied: we fill $20\%$ of each minibatch with the unsafe samples and use double $Q_{V}$ networks for critic update.

\begin{figure}[!t]
	\centering
	\begin{subfigure}[b]{.48\columnwidth}
		\centering
		\includegraphics[width=\linewidth]{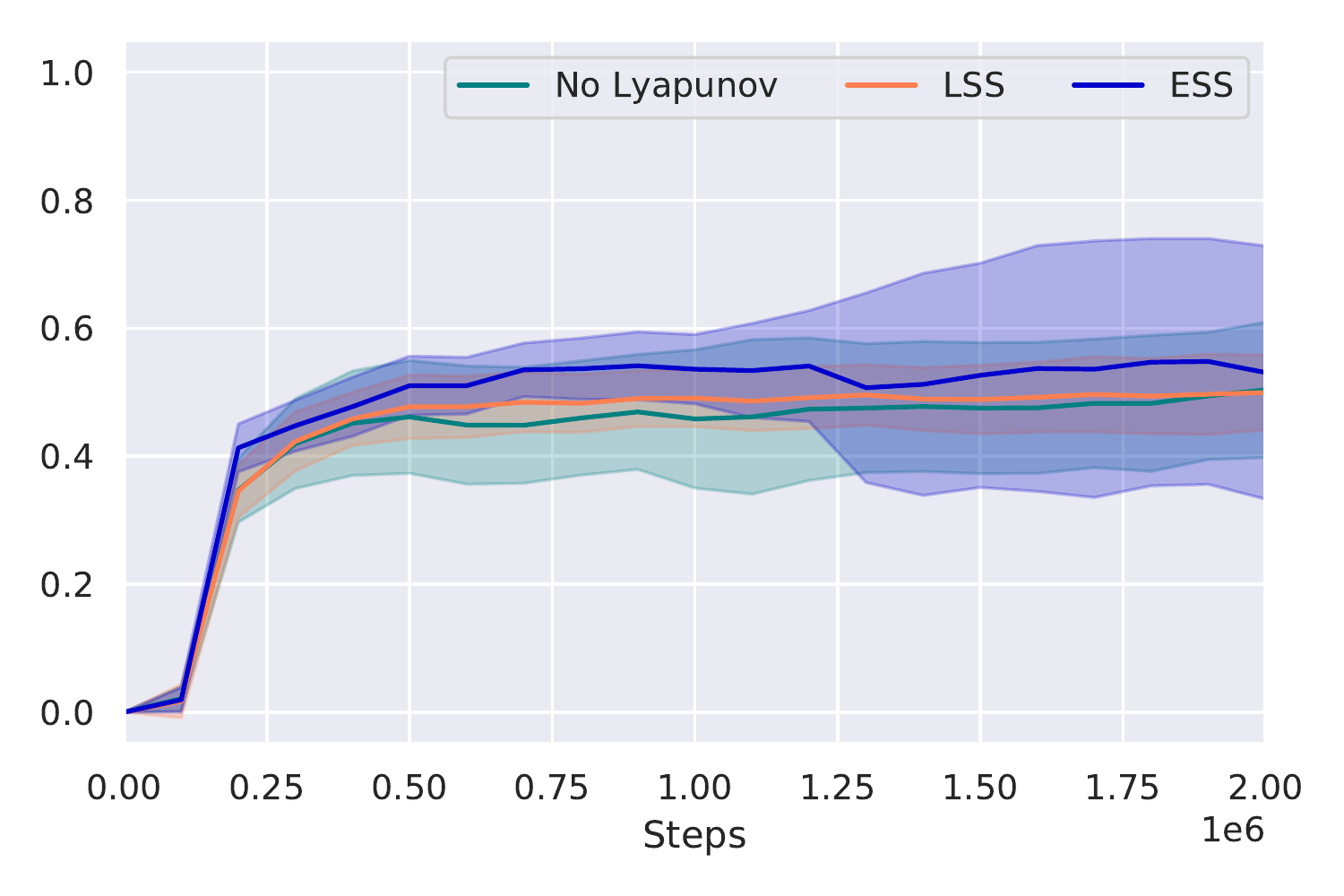}
		\caption{Reacher, $r_{\mathrm{c}}$ (average)}
		\label{fig:reacher_spec_ddpg_correct_average}
	\end{subfigure}
	\hspace*{\fill}
	\begin{subfigure}[b]{.48\columnwidth}
		\centering
		\includegraphics[width=\linewidth]{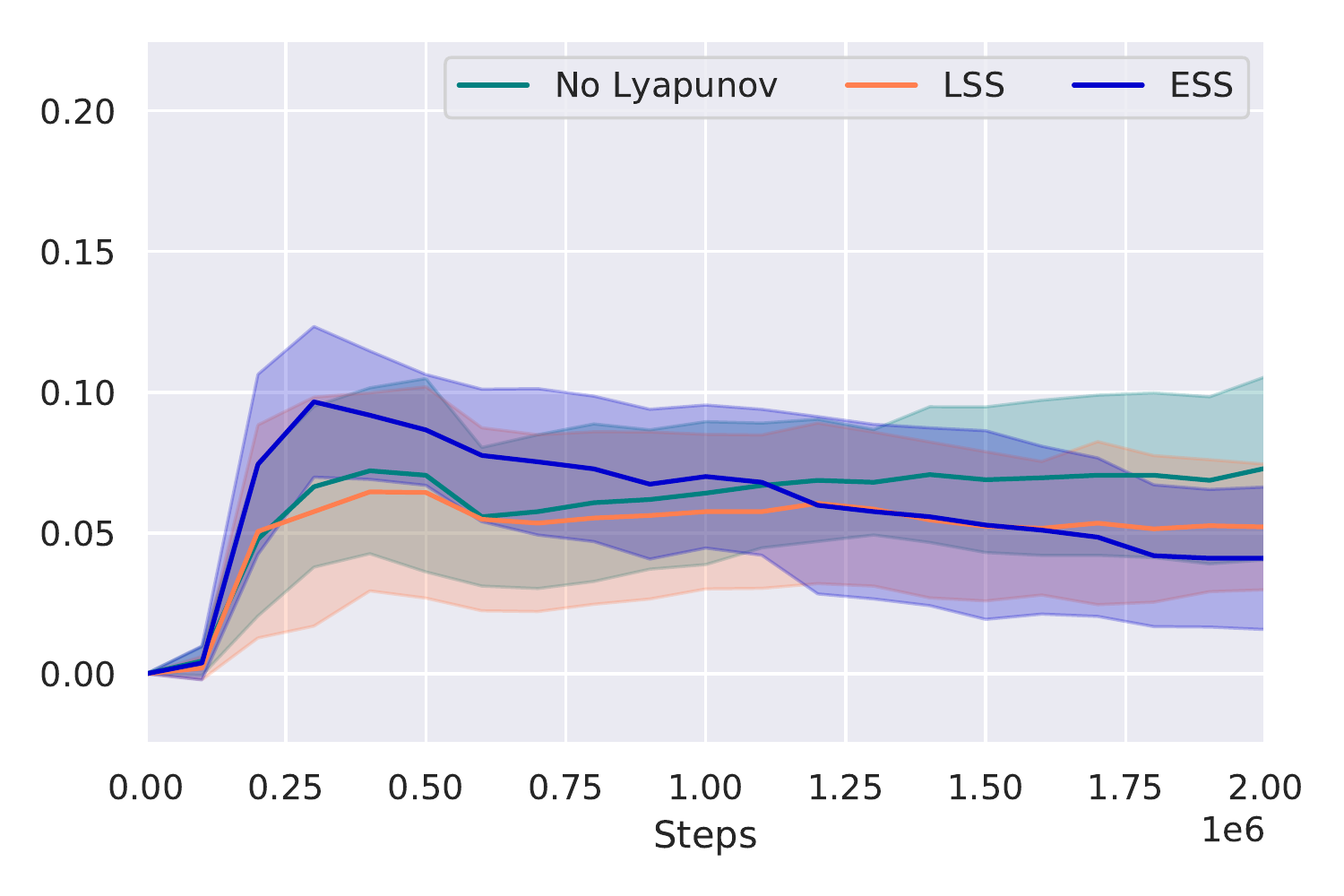}
		\caption{Reacher, $r_{\mathrm{fp}}$ (average)}
		\label{fig:reacher_spec_ddpg_falsepositive_average}
	\end{subfigure}
	\hfil
	\centering
	\begin{subfigure}[b]{.48\columnwidth}
		\centering
		\includegraphics[width=\linewidth]{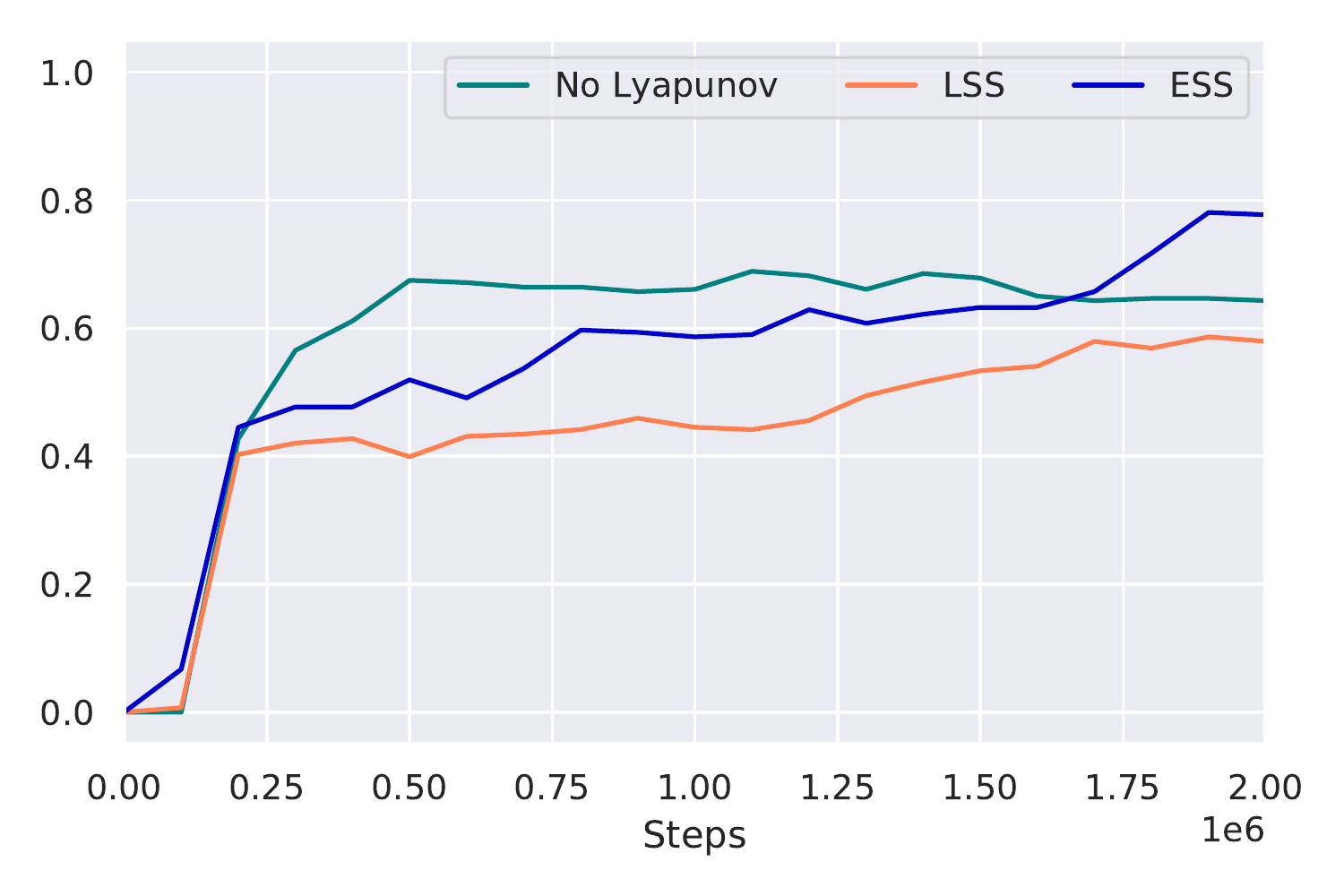}
		\caption{Reacher, $r_{\mathrm{c}}$ (best)}
		\label{fig:reacher_spec_ddpg_correct_best}
	\end{subfigure}
	\hspace*{\fill}
	\begin{subfigure}[b]{.48\columnwidth}
		\includegraphics[width=\linewidth]{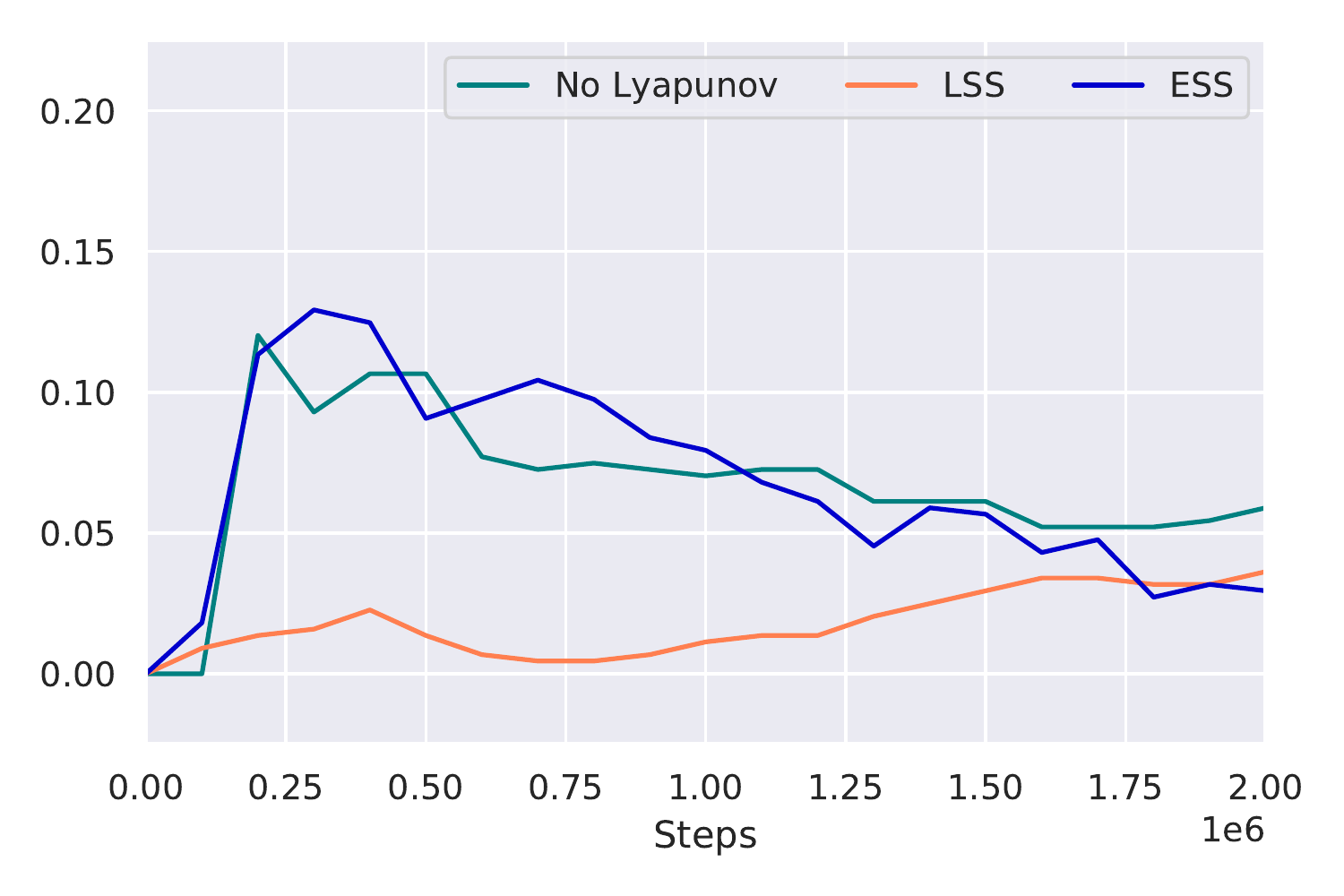}
		\caption{Reacher, $r_{\mathrm{fp}}$ (best)}
		\label{fig:reacher_spec_ddpg_falsepositive_best}
	\end{subfigure}
	\hfil
	\centering
	\begin{subfigure}[b]{.48\columnwidth}
		\centering
		\includegraphics[width=\linewidth]{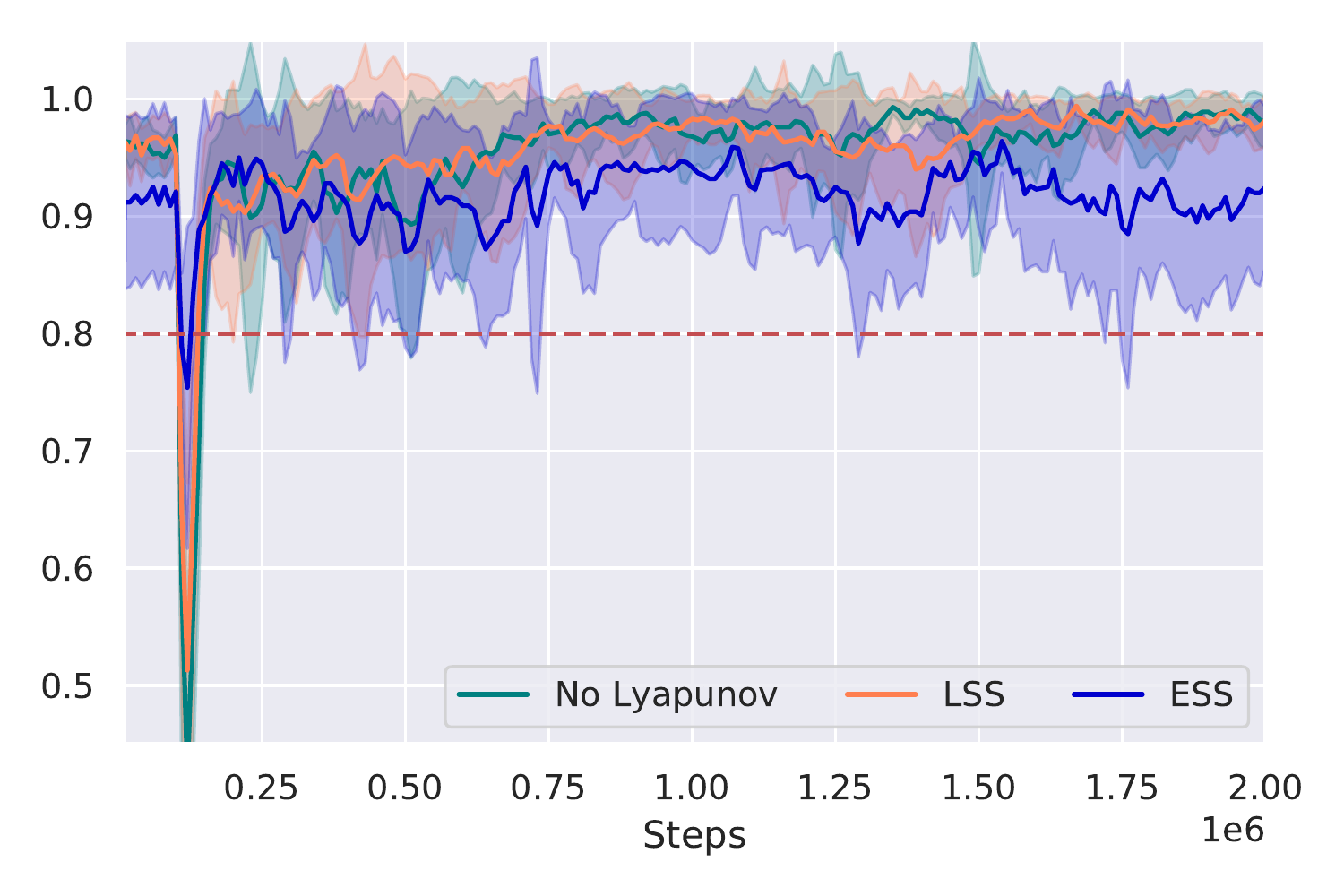}
		\caption{Reacher, AES}
		\label{fig:reacher_spec_ddpg_aesafety}
	\end{subfigure}
	\caption{Safety specification via deep RL tested on the Reacher. (a-b) are the results averaged across 10 random seeds, and (c-d) are the best results for various methods. (e) displays the average episode safety swept across all seeds. Color schemes are equivalent to Fig. \ref{fig:speculation_tab}.} 
	\label{fig:speculation_ddpg}
\end{figure}

The Lyapunov-based RL agents require auxiliary cost~$\epsilon$, as in Section \ref{sec:setup}.
For the case of a continuous state space, the safe set is not explicitly defined, so $\epsilon$ should be approximated.
We first set the denominator of $\epsilon$ to $T^{\pi}(\bm{s}) \approx (1 - \gamma)^{-1}$ to prevent it from being larger than the true value.
To estimate $\min_{\bm{s} \in S^{\ast}(\alpha)}\{ \alpha - V^{\pi}(\bm{s},1)\}$, we use supplementary memory that remembers the value of $\{\alpha - V^{\pi}(\bm{s},1)\}^{+}$ for $\bm{s}$ such that $V^{\pi}(\bm{s},1) \leq \alpha$.s
When an episode is terminated, an agent computes $V^{\pi}$ for all the states in the trajectory and find the maximum among the values that satisfy $V^{\pi}(\bm{s},1) \leq \alpha$.
The memory stores the result for the 100 latest trajectories.

We also exploit the two actors of the ESS actor-critic to ensure safe operation.
Since it takes time to construct a stable exploratory actor, the agent makes stochastic choices between the two actors in the early stages.
The probability of an SS-actor being chosen is 1 at the first gradient step and declines linearly until it becomes 0 after the first half of the learning process.
The SS-actor is also utilized as the backup policy; that is, the agent takes the action using $\pi_{s}$ if the AES is less than the threshold $1-\alpha$, regardless of the policy choice scheme described above.
To reduce computation time, $\lambda_{\omega_{s}}$ is fixed to 0 for the ESS actor-critic.

Fig. \ref{fig:speculation_ddpg} summarizes the experimental result.
We perform tests on 10 random seeds to take an average (\ref{fig:reacher_spec_ddpg_correct_average}, \ref{fig:reacher_spec_ddpg_falsepositive_average}) and to display the ones that attain the greatest $r_{\mathrm{c}}$ among various methods (\ref{fig:reacher_spec_ddpg_correct_best}, \ref{fig:reacher_spec_ddpg_falsepositive_best}).
Comparing the average cases, the ESS actor-critic shows improvement in both specification criteria, and is noticeable for false positives.
ESS consistently reduces $r_{\mathrm{fp}}$ except for the first $3\times10^5$ steps and then achieves $4.10\%$, while the baseline and LSS settle at $7.30\%$ and $5.22\%$, respectively.
The learning curves of ESS and the baseline are similar at the very start, since ESS does not regularly use the exploratory policy then.
The  exploratory policy in ESS supplements novel information about the states, which are normally the elements of the target set, and the safe set thus becomes more accurate.
On the other hand, those of the baseline stay stagnant because the agent barely falls into an unusual trajectory with the SS-policy.
Regarding LSS, we observe that the regularization term in its update rule degrades the overall performance.

As seen by the large confidence interval of ESS in Fig. \ref{fig:reacher_spec_ddpg_correct_average}, the effect of the exploratory policy varies.
ESS  performs as the description in Section \ref{sec:ess};
considering the best cases, ESS attains $77.7\%$ for the correct specification, which is $13.4\%$ above the baseline.
The exploratory policies sometimes converge fast and become indifferent to the SS-policies in terms of exploration, resulting in poor performance.
Note that the difference in ESS behavior is determined by the approximation error in the critic $Q_{V}$.
Although it is difficult to organize the parametrized critic, we can exploit the potential of ESS by running on multiple seeds and finding the best among them.

\begin{figure}[!t]	
	\centering
	 \begin{subfigure}[b]{.27\columnwidth}
		\centering
		\includegraphics[width=\linewidth]{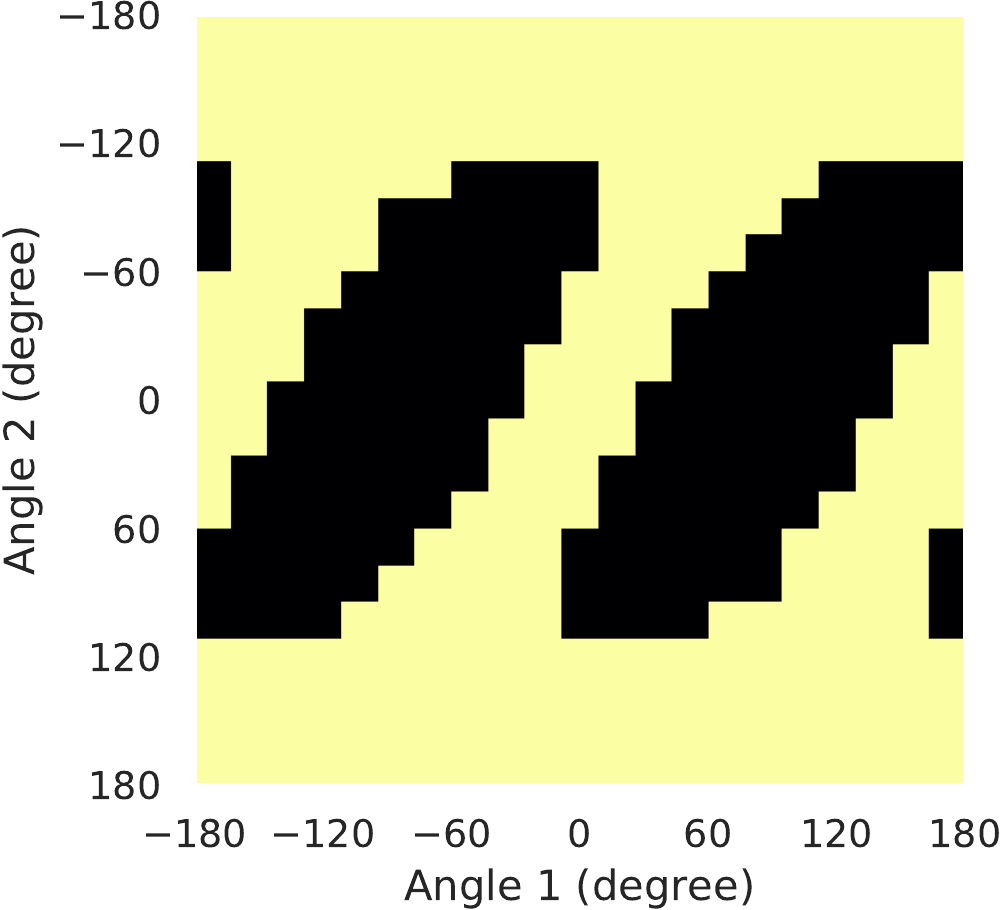}
		\caption{Ground truth}
	\end{subfigure}%
		\hspace*{\fill}
\begin{subfigure}[b]{.23\columnwidth}
		\centering
		\includegraphics[width=\linewidth]{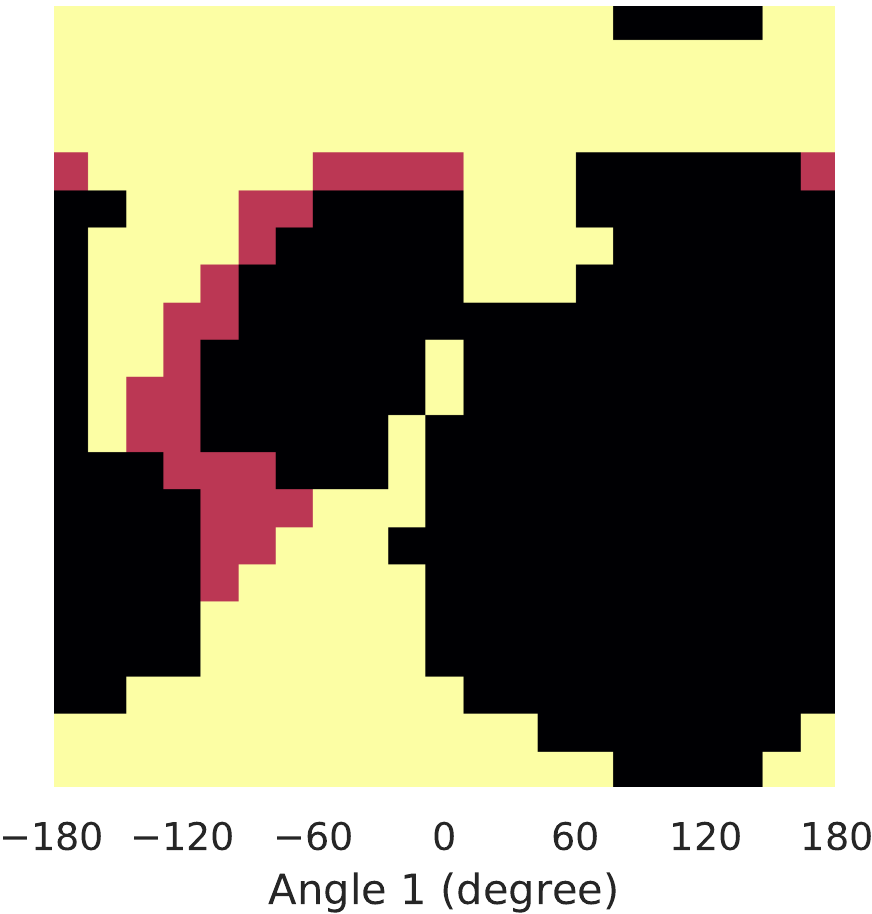}
		\caption{Baseline}
\end{subfigure}	
\begin{subfigure}[b]{.23\columnwidth}
		\centering
		\includegraphics[width=\linewidth]{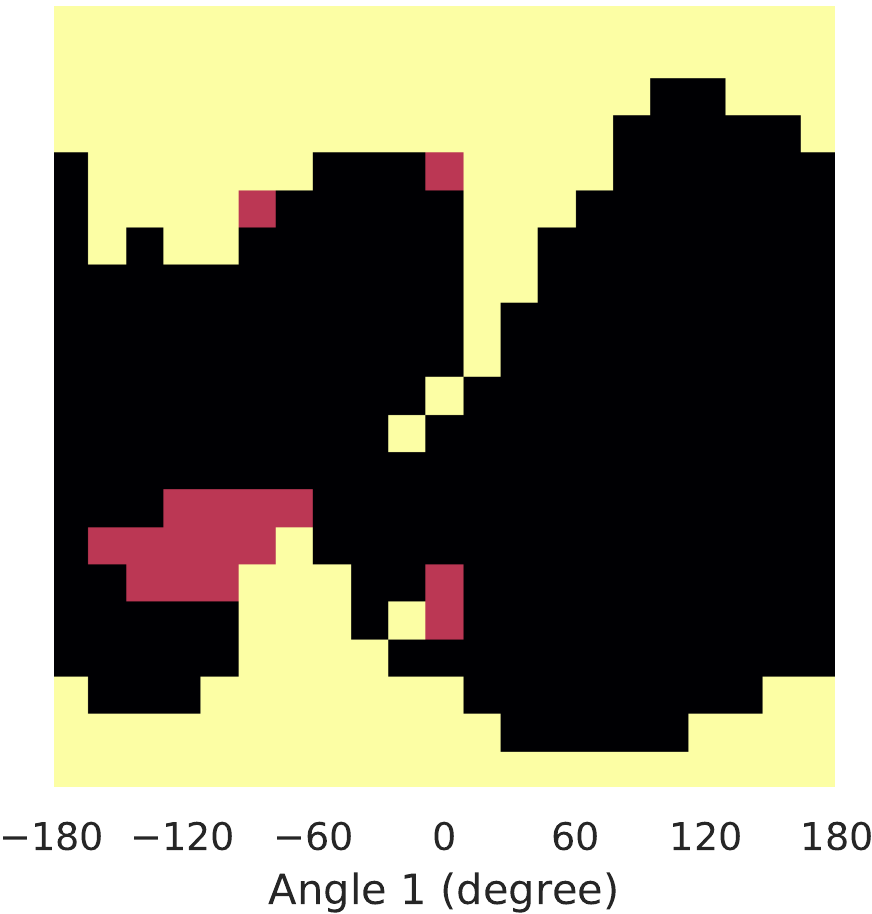}
		\caption{LSS}
\end{subfigure}	
\begin{subfigure}[b]{.23\columnwidth}
		\centering
		\includegraphics[width=\linewidth]{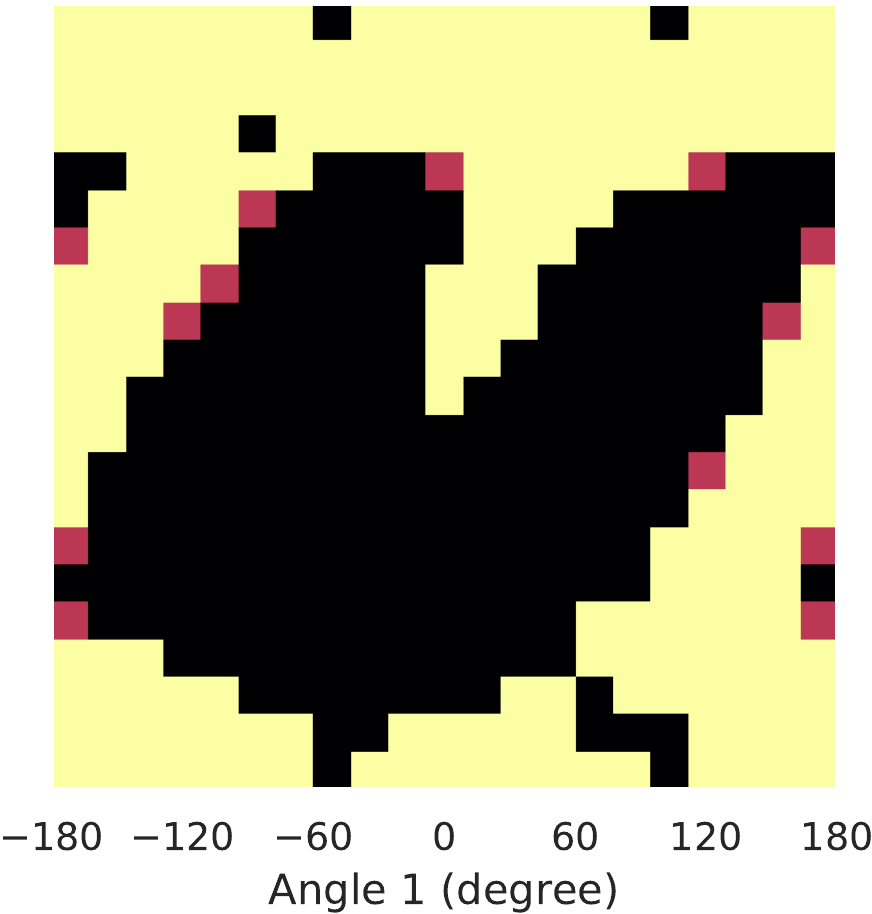}
		\caption{ESS}
\end{subfigure}	
	\caption{Safe sets in the state space of the Reacher. Each grid point denotes a state of the end effector whose position is determined by the angles of the two joints and whose velocity is 0. Given $\alpha = 0.2$, the ground truth $S^{\ast}(\alpha)$ is denoted by yellow in (a). The other figures show the estimated safe set obtained by (b) the baseline, (c) LSS, and (d) ESS. Color schemes are equivalent to Fig. \ref{fig:visualize_tab}.}
	\label{fig:visualize_ddpg}
\end{figure}

In Fig. \ref{fig:visualize_ddpg}, we further visualize a relevant part of the state space and the safe sets in it.
Each grid map displays $\hat{S}^{\pi}(\alpha)$ of the agent whose $r_{\mathrm{c}}$ is the greatest among the 10 random seeds discussed above.
The safe set obtained by ESS  clearly resembles the true safe set better than the others.

\section{Conclusion}


We have proposed a model-free safety specification tool that incorporates a Lyapunov-based safe RL approach with probabilistic reachability analysis.
Our method exploits the Lyapunov constraint to construct an exploratory policy that mitigates the discrepancy between state distributions of the experience replay (or the tabular Q-function) and the environment.
Another salient feature of 
the proposed method is that it can be implemented on generic, model-free deep RL algorithms, particularly in continuous state and action spaces through Lagrangian relaxation.
The results of our experiments demonstrate that our method encourages visiting the unspecified states, thereby improving the accuracy of specification. 
By bridging probabilistic reachability analysis and reinforcement learning, this work can provide an exciting avenue for future research in terms of extensions to partially observable MDPs, and  model-based exploration and its regret analysis, among others.


\section*{Appendix}

\subsection{Deep RL Implementation}\label{appendix:deeprl}

In this section, we provide a specific description of the deep RL agents used in our experiments.
Table \ref{tab:ddpg_top_layer} displays the basic architecture of neural networks, all of which are fully connected and consist of two hidden layers with ReLU as an activation function unless it is an estimator of $Q_{V}$.
The first and second hidden layers have 400 and 300 nodes, respectively.
Adam optimizer \cite{kingma2015adam} is used to apply gradient descent.
Aside from the techniques stated in Section \ref{sec:deeprl_result}, an action is perturbed with Ornstein-Uhlenbeck noise with parameters $\mu = 0$, $\theta = 0.1$, and $\sigma = 0.05$.

\begin{table}[!h]
	\centering
	\begin{tabular}{ llll }
		\toprule
		Type			& Output size		&	Activation              &   Learning rate\\
		\midrule
		Critic			& 1					&	$\mathrm{clamp}(0,1)$   &   $10^{-4}$\\
		Actor			& $\mathrm{dim}(a)$	&	tanh                    &   $10^{-5}$\\
		$\log\lambda$	& 1					&	$\mathrm{clamp}(-10,6)$ &   $10^{-6}$\\
		\bottomrule
	\end{tabular}
	\caption{The top layers of respective networks in DDPG.}
	\label{tab:ddpg_top_layer}
\end{table}

\subsection{Environments}\label{appendix:envs}

An environment provides a Boolean \texttt{done} signal that declares the termination of an episode strictly equivalent to $\bm{1}_{\mathcal{S}_{\mathrm{term}}}$.
When its value is 1, both $Q_V$ and $Q_T$ at that state are set to 0. If the length of an episode exceeds the time limit before arriving at a terminal state, the environment resets itself, but \texttt{done} is still 0 at that moment.
Refer to Table \ref{tab:env_params} for the time limit and the discount factor settings.

\textbf{Randomized integrator.}~
A vanilla double integrator is a system with a 2D state $(x_{1},x_{2})$ and the scalar control $u$.
$x_{1}$ and $x_{2}$ represent the position and velocity on a 1D line, respectively. The control is an acceleration.

We add a few features to construct a safety specification problem in this environment.
First, we set the terminal states as the points near the origin $(x_{1},x_{2}) \in [-0.2,0.2] \times [-3.75 \times 10^{-3}, 3.75 \times 10^{-3}]$.
Next the target set is defined as all the states
$(x_1, x_2) \notin [-1,1] \times [-0.5, 0.5]$.
Finally, we restrict admissible action to the range $[-0.5,0.5]$, and adjust the dynamics so that the acceleration is scaled to $0.5u/|u|$ with probability $1/2$.
Due to the introduction of stochastic behavior, it becomes more difficult to reach the terminal states safely than in the original environment.

\begin{figure}[!t]
	\centering
	\includegraphics[width=0.85\columnwidth]{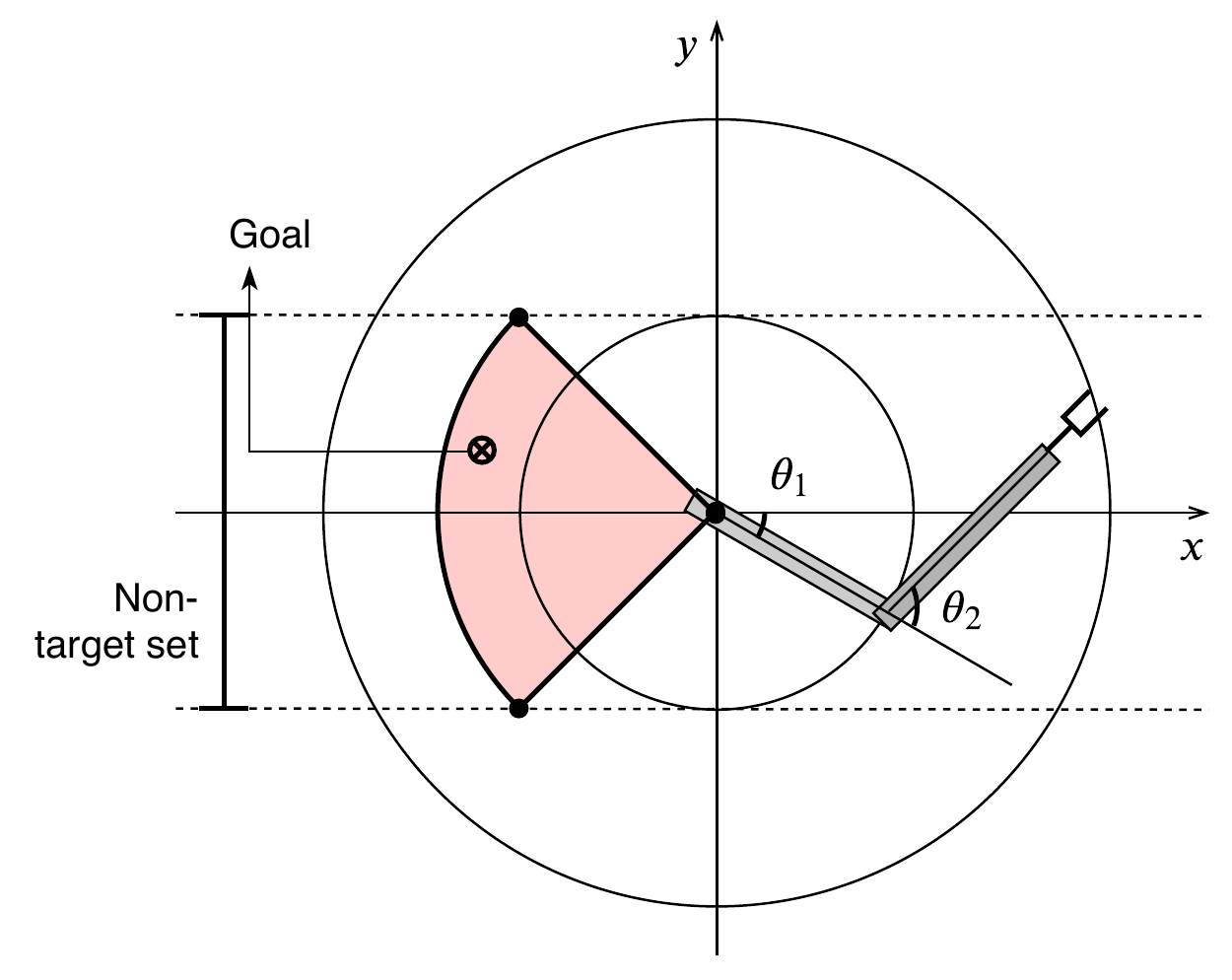}
	\caption{Description of the Reacher environment.}
	\label{fig:reacher}
\end{figure}

\textbf{Reacher.}~
Reacher is a simulative planar 2-DOF robot with two arms attached to joints implemented with a Mujoco engine \cite{mujoco}.
The joint of the first arm is fixed on the center of the plane, and the joint of the second is connected to the movable end of the first.
The objective of the robot is to touch a motionless goal point with its end effector.
An observation is thus defined as a vector that contains the angular positions and the angular velocities of the joints as well as the position of the goal.
The action is defined as the torques on the joints, each of which is bounded in the range $[-1,1]$.

Let the coordinates be defined as in Fig. \ref{fig:reacher}.
Specifically, the goal is deployed randomly in the hued area $\{(x,y) | \sqrt{x^{2}+y^{2}} \leq \sqrt{2}l, |\arctan{y/x}| \leq \pi/4 \}$, where $l$ is the length of one arm.
The exact position changes for each reset.
We define the target set as $\{(x,y) | |y| > l\}$, where $(x,y)$ is the coordinate of the tip. 

We derive the probabilistic safe set in Fig. \ref{fig:visualize_ddpg} under the assuming no friction.
This is not the case in a Mujoco-based simulation, but the effect of such an assumption is minor.
Recall that the states displayed in Fig. \ref{fig:visualize_ddpg} stand for an end effector with zero velocity.
If appropriate control is applied, the robot can avoid reaching the target set by moving toward an arbitrary position near the goal unless it launched from the target set at the beginning. 

In our simulation studies, we only assess the agents with the states where the goal point is given by $(-2l,0)$, and the angular velocity is $(\dot{\theta}_{1},\dot{\theta}_{2}) = (0,0)$.
We use the Reacher configuration provided by Gym \cite{openaigym}.

\begin{table}[!h]
	\centering
	\begin{tabular}{ lll }
		\toprule
		Environment		& Time limit	&	$\gamma$\\
		\midrule
		Integrator		& $1000$		&	$1 - 10^{-4}$\\
		Reacher			& $300$			&	$1 - 10^{-3}$\\
		\bottomrule
	\end{tabular}
	\caption{The environment-specific parameters.}
	\label{tab:env_params}
\end{table}


\bibliographystyle{IEEEtran}
\bibliography{refered_method,related_work}

\begin{thebibliography}{10}
\providecommand{\url}[1]{#1}
\csname url@samestyle\endcsname
\providecommand{\newblock}{\relax}
\providecommand{\bibinfo}[2]{#2}
\providecommand{\BIBentrySTDinterwordspacing}{\spaceskip=0pt\relax}
\providecommand{\BIBentryALTinterwordstretchfactor}{4}
\providecommand{\BIBentryALTinterwordspacing}{\spaceskip=\fontdimen2\font plus
\BIBentryALTinterwordstretchfactor\fontdimen3\font minus
  \fontdimen4\font\relax}
\providecommand{\BIBforeignlanguage}[2]{{%
\expandafter\ifx\csname l@#1\endcsname\relax
\typeout{** WARNING: IEEEtran.bst: No hyphenation pattern has been}%
\typeout{** loaded for the language `#1'. Using the pattern for}%
\typeout{** the default language instead.}%
\else
\language=\csname l@#1\endcsname
\fi
#2}}
\providecommand{\BIBdecl}{\relax}
\BIBdecl

\bibitem{abate2008probabilistic}
A.~Abate, M.~Prandini, J.~Lygeros, and S.~Sastry, ``Probabilistic reachability
  and safety for controlled discrete time stochastic hybrid systems,''
  \emph{Automatica}, vol.~44, no.~11, pp. 2724--2734, Nov 2008.

\bibitem{Majumdar2014}
A.~Majumdar, R.~Vasudevan, M.~M. Tobenkin, and R.~Tedrake, ``Convex
  optimization of nonlinear feedback controllers via occupation measures,''
  \emph{Int. J. Robot. Res.}, vol.~33, no.~9, pp. 1209--1230, 2014.

\bibitem{Chen2018}
M.~Chen, S.~L. Herbert, M.~S. Vashishtha, S.~Bansal, and C.~J. Tomlin, ``A
  general system decomposition method for computing reachable sets and tubes,''
  \emph{IEEE Trans. Automat. Contr.}, vol.~63, no.~11, pp. 3675--3688, 2018.

\bibitem{gillula2011applications}
J.~H. Gillula, G.~M. Hoffmann, H.~Huang, M.~P. Vitus, and C.~J. Tomlin,
  ``Applications of hybrid reachability analysis to robotic aerial vehicles,''
  \emph{Int. J. Robot. Res.}, vol.~30, no.~3, pp. 335--354, 2011.

\bibitem{Piovan2015}
G.~Piovan and K.~Byl, ``Reachability-based control for the active slip model,''
  \emph{Int. J. Robot. Res.}, vol.~34, no.~3, pp. 270--287, 2015.

\bibitem{Malone2017}
N.~Malone, H.~T. Chiang, K.~Lesser, M.~Oishi, and L.~Tapia, ``Hybrid dynamic
  moving obstacle avoidance using a stochastic reachable set-based potential
  field,'' \emph{IEEE Trans. Robot.}, vol.~33, no.~5, pp. 1124--1138, 2017.

\bibitem{Fisac2018}
J.~F. Fisac, A.~K. Akametalu, M.~N. Zeilinger, S.~Kaynama, J.~Gillula, and
  C.~J. Tomlin, ``A general safety framework for learning-based control in
  uncertain robotic systems,'' \emph{IEEE Trans. Automat. Contr.}, vol.~64,
  no.~7, pp. 2737--2752, 2018.

\bibitem{Fisac2019}
J.~F. Fisac, N.~F. Lugovoy, V.~Rubies-Royo, S.~Ghosh, and C.~Tomlin, ``Bridging
  {Hamilton-Jacobi} safety analysis and reinforcement learning,'' in
  \emph{Proc. IEEE Int. Conf. Robot. Autom.}, 2019, pp. 8550--8556.

\bibitem{Berkenkamp2017}
F.~Berkenkamp, M.~Turchetta, A.~P. Schoellig, and A.~Krause, ``Safe model-based
  reinforcement learning with stability guarantees,'' in \emph{Adv. Neural Inf.
  Process. Syst.}, 2017.

\bibitem{Richards2018}
S.~M. Richards, F.~Berkenkamp, and A.~Krause, ``The {Lyapunov} neural network:
  Adaptive stability certification for safe learning of dynamical systems,'' in
  \emph{Proc. 2nd Conf. on Robot Learn.}, 2018, pp. 466--476.

\bibitem{Wang2018}
L.~Wang, E.~A. Theodorou, and M.~Egerstedt, ``Safe learning of quadrotor
  dynamics using barrier certificates,'' in \emph{Proc. IEEE Int. Conf. Robot.
  Autom.}, 2018, pp. 2460--2465.

\bibitem{Taylor2019}
A.~Taylor, A.~Singletary, Y.~Yue, and A.~Ames, ``Learning for safety-critical
  control with control barrier functions,'' \emph{arXiv preprint
  arXiv:1912.10099}, 2019.

\bibitem{Summers2010}
S.~Summers and J.~Lygeros, ``Verification of discrete time stochastic hybrid
  systems: A stochastic reach-avoid decision problem,'' \emph{Automatica},
  vol.~46, no.~12, pp. 1951--1961, 2010.

\bibitem{Lesser2016}
K.~Lesser and M.~Oishi, ``Approximate safety verification and control of
  partially observable stochastic hybrid systems,'' \emph{IEEE Trans. Automat.
  Contr.}, vol.~62, no.~1, pp. 81--96, 2017.

\bibitem{Yang2018}
I.~Yang, ``A dynamic game approach to distributionally robust safety
  specifications for stochastic systems,'' \emph{Automatica}, vol.~94, pp.
  94--101, 2018.

\bibitem{chow2018lyapunov}
Y.~Chow, O.~Nachum, E.~Duenez-Guzman, and M.~Ghavamzadeh, ``A {L}yapunov-based
  approach to safe reinforcement learning,'' in \emph{Adv. Neural Inf. Process.
  Syst.}, 2018, pp. 8103--8112.

\bibitem{turchetta2016safe}
M.~Turchetta, F.~Berkenkamp, and A.~Krause, ``Safe exploration in finite markov
  decision processes with gaussian processes,'' in \emph{Adv. Neural Inf.
  Process. Syst.}, 2016, pp. 4312--4320.

\bibitem{alshiekh2018safe}
M.~Alshiekh, R.~Bloem, R.~Ehlers, B.~K{\"o}nighofer, S.~Niekum, and U.~Topcu,
  ``Safe reinforcement learning via shielding,'' in \emph{Proc. {AAAI} Conf. on
  Artif. Intell.}, 2018, pp. 2669--2678.

\bibitem{wachi2018safe}
A.~Wachi, Y.~Sui, Y.~Yue, and M.~Ono, ``Safe exploration and optimization of
  constrained {MDPs} using {G}aussian processes,'' in \emph{Proc. {AAAI} Conf.
  on Artif. Intell.}, 2018, pp. 6548--6556.

\bibitem{tsitsiklis1994asynchronous}
J.~N. Tsitsiklis, ``Asynchronous stochastic approximation and q-learning,''
  \emph{Mach. Learn.}, vol.~16, no.~3, pp. 185--202, 1994.

\bibitem{Bertsekas1999}
D.~P. Bertsekas, \emph{\BIBforeignlanguage{eng}{Nonlinear Programming}},
  2nd~ed.\hskip 1em plus 0.5em minus 0.4em\relax Belmont, MA, USA: Athena
  Scientific, 1999.

\bibitem{bohez2019value}
S.~Bohez, A.~Abdolmaleki, M.~Neunert, J.~Buchli, N.~Heess, and R.~Hadsell,
  ``Value constrained model-free continuous control,'' \emph{arXiv preprint
  arXiv:1902.04623}, 2019.

\bibitem{hasselt2010double}
H.~van Hasselt, ``Double {Q}-learning,'' in \emph{Adv. Neural Inf. Process.
  Syst.}, 2010, pp. 2613--2621.

\bibitem{duan2016benchmarking}
Y.~Duan, X.~Chen, R.~Houthooft, J.~Schulman, and P.~Abbeel, ``Benchmarking deep
  reinforcement learning for continuous control,'' in \emph{Int. Conf. on Mach.
  Learn.}, vol.~48, 2016, pp. 1329--1338.

\bibitem{lillicrap2015continuous}
T.~P. Lillicrap, J.~J. Hunt, A.~Pritzel, N.~Heess, T.~Erez, Y.~Tassa,
  D.~Silver, and D.~Wierstra, ``Continuous control with deep reinforcement
  learning,'' in \emph{arXiv preprint arXiv:1509.02971}, 2015.

\bibitem{kingma2015adam}
D.~P. Kingma and J.~Ba, ``Adam: {A} method for stochastic optimization,'' in
  \emph{arXiv preprint arXiv:1412.6980}, 2014.

\bibitem{mujoco}
E.~{Todorov}, T.~{Erez}, and Y.~{Tassa}, ``Mujoco: A physics engine for
  model-based control,'' in \emph{2012 IEEE/RSJ International Conference on
  Intelligent Robots and Systems}, 2012, pp. 5026--5033.

\bibitem{openaigym}
G.~Brockman, V.~Cheung, L.~Pettersson, J.~Schneider, J.~Schulman, J.~Tang, and
  W.~Zaremba, ``Open{AI} {G}ym,'' \emph{arXiv preprint arXiv:1606.01540}, 2016.

\end{thebibliography}



\end{document}